\documentclass[a4letter,11pt]{article}
\usepackage{a4wide}

\usepackage{environ}
\usepackage[margin=1in]{geometry}
\usepackage{enumerate}
\usepackage[usenames,dvipsnames]{xcolor}
\definecolor{Gred}{RGB}{219, 50, 54}
\definecolor{Ggreen}{RGB}{60, 186, 84}
\definecolor{Gblue}{RGB}{72, 133, 237}
\definecolor{Gyellow}{RGB}{247, 178, 16}
\definecolor{ToCgreen}{RGB}{0, 128, 0}
\definecolor{myGold}{RGB}{231,141,20}
\definecolor{myBlue}{rgb}{0.19,0.41,.65}
\definecolor{myPurple}{RGB}{175,0,124}
\definecolor{niceRed}{RGB}{153,0,0}
\definecolor{niceRed}{RGB}{190,38,38}
\definecolor{blueGrotto}{HTML}{059DC0}
\definecolor{royalBlue}{HTML}{057DCD}
\definecolor{navyBlueP}{HTML}{0B579C}
\definecolor{limeGreen}{HTML}{81B622}
\definecolor{nicePink}{RGB}{247,83,148}

\usepackage{cmap}
\usepackage[T1]{fontenc}
\usepackage{bm}
\usepackage{amsmath,amsfonts,amssymb,amsbsy,amsthm}
\usepackage{bbm}
\usepackage{thm-restate}

\usepackage{graphicx, ucs}

\usepackage{algorithm, caption}
\usepackage[noend]{algpseudocode}
\usepackage{listings}

\usepackage{enumitem}

\usepackage[pagebackref]{hyperref}
\hypersetup{
	colorlinks = true,
	urlcolor = {royalBlue},
	linkcolor = {niceRed},
	citecolor = {royalBlue}
}
\usepackage[nameinlink]{cleveref}

\usepackage{multirow}
\usepackage{array}

\def\compactify{\itemsep=0pt \topsep=0pt \partopsep=0pt \parsep=0pt}
\let\latexusecounter=\usecounter

\definecolor{myC}{rgb}{0, 255, 255}
\definecolor{myY}{rgb}{204, 204, 0}
\definecolor{myM}{rgb}{255, 0, 255}
\definecolor{secinhead}{RGB}{249,196,95}
\definecolor{lgray}{gray}{0.8}

\newtheorem{theorem}{Theorem} 
\newtheorem*{theorem*}{Theorem} 
\newtheorem*{proposition*}{Proposition} 
\newtheorem{lemma}[theorem]{Lemma}

\newtheorem{definition}[theorem]{Definition}

\renewcommand{\Pr}{\mathop{\bf Pr\/}}
\newcommand{\E}{\mathop{\bf E\/}}

\newcommand{\Var}{\mathop{\bf Var\/}}

\newcommand{\reals}{\mathbb R}

\newcommand{\nats}{\mathbb N}

\newcommand{\eps}{\epsilon}

\newcommand{\calA}{\mathcal{A}}

\newcommand{\calF}{\mathcal{F}}

\newcommand{\calN}{\mathcal{N}}

\def\<{\langle}
\def\>{\rangle}

\newcommand{\emm}{\mathrm{e}}

\DeclareMathOperator*{\argmin}{argmin}

\def\wh{\widehat}

\usepackage{amsthm,amsmath,amssymb,amsfonts,amssymb,mathtools}
\usepackage{xcolor}
\usepackage{empheq}
\usepackage{dsfont}
\usepackage{mathrsfs}
\usepackage{graphicx}
\usepackage{enumitem}
\usepackage{subfig}

\usepackage{tikz}
\usetikzlibrary{patterns}
\usetikzlibrary{arrows,shapes,automata,backgrounds,petri,positioning}
\usetikzlibrary{shadows}
\usetikzlibrary{calc}
\usetikzlibrary{spy}
\usepackage{pgf,pgfplots}
\usetikzlibrary{angles, quotes}
\usepackage{pgfmath,pgffor}
\pgfplotsset{compat=1.17}
\usepackage{framed}

\definecolor[named]{ACMBlue}{cmyk}{1,0.1,0,0.1}
\definecolor[named]{ACMYellow}{cmyk}{0,0.16,1,0}
\definecolor[named]{ACMOrange}{cmyk}{0,0.42,1,0.01}
\definecolor[named]{ACMRed}{cmyk}{0,0.90,0.86,0}
\definecolor[named]{ACMLightBlue}{cmyk}{0.49,0.01,0,0}
\definecolor[named]{ACMGreen}{cmyk}{0.20,0,1,0.19}
\definecolor[named]{ACMPurple}{cmyk}{0.55,1,0,0.15}
\definecolor[named]{ACMDarkBlue}{cmyk}{1,0.58,0,0.21}

\newcommand{\lp}{\left}
\newcommand{\rp}{\right}

\def\dkl{\mathrm{KL}}

\newcommand{\R}{\mathbb{R}}

\def\diag{\mathrm{diag}}

\title{Learning Hard-Constrained Models with One Sample\thanks{A preliminary version of the paper will appear in the Proceedings of the 2024 Annual ACM-SIAM Symposium on
Discrete Algorithms ({SODA\emph{'24}}). For the purpose of Open Access, the authors have applied a CC BY public copyright licence to any
Author Accepted Manuscript version arising from this submission. All data is provided in full in the results
section of this paper.}}
\author{%
  \textbf{Andreas Galanis}\footnote{\texttt{\color{magenta}andreas.galanis@cs.ox.ac.uk}
}\\
University of Oxford\\
\and 
\textbf{Alkis Kalavasis}\footnote{\texttt{\color{magenta}alvertos.kalavasis@yale.edu}
}\\
Yale University\\
  \and
    \textbf{Anthimos Vardis Kandiros}\footnote{\texttt{\color{magenta}kandiros@mit.edu}
} \\
    MIT\\
}

\begin{document}

\maketitle

\begin{abstract}
We consider the problem of single-sample learning, i.e., the problem of estimating the parameters of a distribution (Markov Random Field) using a single sample. Previous work has mainly focused on soft-constrained distributions, such as the Ising model, where remarkably it has been shown that efficient single-sample learning is always possible for sparse graphs. 

Here, we focus instead on single-sample learning of \emph{hard}-constrained distributions. As our main running examples, we use the $k$-SAT, the proper coloring models and the general $H$-coloring models (graph homomorphisms), for which we obtain both positive and negative results. In contrast to the soft-constrained case, we show in particular that single-sample estimation is not always possible, and that the existence of an estimator is related to the existence of non-satisfiable instances.

Following the approach of Chatterjee for soft-constrained distributions (Annals of Statistics, 2007), our algorithms are based on the classical pseudo-likelihood estimator. The main difficulty in our setting is that the accuracy of the estimator is not driven by the sufficient statistics of the distribution (which facilitated previous analyses). Instead, we show variance bounds for the estimator using more involved coupling techniques inspired, in the case of $k$-SAT, by Moitra's  sampling algorithm (JACM, 2019); our positive results for single-sample learning for proper $q$-colorings and $H$-colorings build on this new coupling approach. Our impossibility results are based on constructing satisfiable instances which are either uniquely satisfiable or enforce with high probability a particular value to the statistic determining the distribution; in both cases, the distribution becomes insensitive to the underlying parameters.

 Using these methods, in the case of $q$-colorings on graphs with maximum degree $d$, we show that a linear-time estimator exists when $q>d+1$, whereas the problem is non-identifiable when $q\leq d+1$. For general $H$-colorings, we show  that  standard conditions that guarantee sampling, such as Dobrushin's condition, are insufficient for one-sample learning; on the positive side, we provide  a  condition that is sufficient to guarantee linear-time learning and obtain applications for proper colorings and permissive models.  For the $k$-SAT model on formulas with maximum degree $d$, the picture for single-sample learning is more intriguing, and we show a linear-time estimator when $k\gtrsim 6.45\log d$, whereas the problem becomes non-identifiable when $k\lesssim \log d$.

\end{abstract}
\thispagestyle{empty}
\setcounter{page}{0}
\newpage

\section{Introduction}

Markov Random Fields (MRFs) are a standard framework for studying high-dimensional distributions with conditional independence structure, represented via a graph structure (typically an undirected graph or hypergraph). 
The most well-studied example of an MRF is the Ising model which, in its simplest form, is parameterised by a real number $\beta$ (corresponding to the so-called \emph{inverse temperature} in statistical physics). Given a graph $G=(V,E)$ with $n=|V|$, the Ising model corresponds to the probability distribution $\mu_{G,\beta}$ on all possible assignments $x \in \{-1,1\}^n$, with $\mu_{G,\beta}(x) \propto\exp(\beta x^\top J x)$, where  $J$ is the adjacency matrix of $G$.

There has been an extensive amount of work on undestanding the complexity of sampling and learning of MRFs. In the particular case of the Ising model, there is an excellent understanding of the computational complexity of sampling, see \cite{sly2012computational,galanis2016inapproximability,eldan2022spectral,chen2022localization,el2022sampling,anari2022entropic,koehler2022sampling,chen2023rapid}. Analogously, 
in the computational learning theory community, there have been major developments in understanding the efficiency of various statistical tasks, including hypothesis testing \cite{daskalakis2019testing,gheissari2018concentration} and learning from multiple samples \cite{
bresler2015efficiently,
vuffray2016interaction,klivans2017learning}. 

Starting from the seminal work of Chatterjee \cite{chatterjee2007estimation}, there has been significant interest in single-sample learning. In the context of single-sampling learning for the Ising model, we are given the graph $G$ and a single sample $\sigma$ from $\mu_{G,\beta}$, and the goal is to learn the value of $\beta$. Surprisingly, on bounded-degree graphs $G$, single-sample learning is always possible for any value of $\beta$ provided that $|\beta|=O(1)$; more precisely,  it is possible to compute $\hat{\beta}$ such that $|\hat{\beta}-\beta|=O( 1/\sqrt{n})$ with high probability over the choice of the random sample $x$, see~\cite{chatterjee2007estimation, ghosal2020joint} and \cite{bhattacharya2018inference,dagan2020estimating} for various refinements/generalisations of this result for multi-parametered Ising distributions. These positive results come in sharp contrast with known results in the literature about the complexity of sampling,  where phase transition phenomena render standard sampling approaches inefficient (such as Glauber dynamics) or even  make approximate sampling NP-hard (when $\beta<0$), see   \cite{sly2012computational,galanis2016inapproximability}. All the aforementioned results for single-sample learning have mainly focused on variants of the Ising model
\cite{chatterjee2007estimation,dagan2020estimating,daskalakis2019regression,mukherjee2021high,mukherjee2020estimation}, which are  \textbf{soft}-constrained, i.e., every possible assignment $x\in \{-1,1\}^n$ has non-zero weight in the underlying distribution. 

In this paper, we consider instead the case of \textbf{hard}-constrained distributions where some assignments have zero weights. 
Canonical examples in this class are the $k$-SAT and proper $q$-colorings models which have been extensively studied in terms of sampling/inference (see \Cref{sec:relatedwork}). Our main results show that for both of these problems the picture for single-sample learning is more intriguing (compared to what one might expect from prior work on soft-constrained learning or sampling), switching from efficiently solvable to impossible (unidentifiable) depending on the values of $k,q$ relative to the degree $d$. We also investigate the more general setting of $H$-colorings where we show impossibility results for single-sample learning in regimes where sampling is possible. We nevertheless propose a natural condition under which efficient estimation is possible and apply it to obtain estimators for a large class of models.

\subsection{Results for $k$-SAT}

In the vanilla version of the $k$-SAT model, 
the MRF distribution is over the set of  the satisfying assignments of a given $k$-SAT formula $\Phi$; here, we consider the variant (see, e.g., \cite{achlioptas2003threshold}) where one of the truth assignments is biased according to a parameter $\beta$. More precisely, let $\Phi=\Phi_{n,k,d}$ denote a $k$-SAT formula with $n$ variables, where every clause contains $k$ distinct literals and every variable is contained in at most $d$ clauses. We identify $1$ with \texttt{True} and $0$ with \texttt{False} and denote by $\Omega = \Omega(\Phi) \subseteq \{0,1\}^n$ the set of satisfying  assignments of $\Phi$ (i.e., $\Omega$ consists of all the assignments to the variables that satisfy all the clauses of $\Phi$). Then, for a parameter $\beta \in \reals$, the MRF for the $k$-SAT model gives to each $\sigma\in \Omega$ a weight that depends on the number of variables set to 1, specifically
\begin{equation}\label{eq:hard_sat}
\Pr_{\Phi,\beta
}[\sigma] \propto  \emm^{\beta C(\sigma)} \mbox{ for all $\sigma\in \Omega(\Phi)$},
\end{equation}
where $C(\sigma)$ is the number of variables set to 1 under $\sigma$. 
This tilted SAT model has been studied in more applied settings as ``weighted model counting'' \cite{chavira2008probabilistic,Sang2005PerformingBI,modelCounting2014,Chakraborty2015FromWT,articleModelCount}. 
We will refer to the distribution in \eqref{eq:hard_sat} as the \emph{hard-SAT} distribution\footnote{There is also a ``soft''-SAT distribution, known also as permissive SAT, which is over the space of all possible assignments $\{0,1\}^n$, where the weight of an assignment depends on the number of unsatisfied clauses. The soft-version can viewed as a version of the tensor Ising model and is therefore covered by the results of \cite{mukherjee2020estimation}.
}.  

A recent breakthrough of Moitra \cite{moitra2019approximate} gave an algorithm for approximately sampling from the degree-$d$ hard-SAT model (when $\beta=0$) when $k \geq 60\log d$, by exploiting in a clever way connections with the Lovász Local Lemma (LLL)\footnote{This efficient sampling threshold has been improved in subsequent works
\cite{feng2021fast,feng2021sampling,jain2022towards, detcounting}. See \cite{detcounting} for a reference which achieves a fast uniform sampler when $k \gtrsim 5\log d$.}. 

Our first main result is a single-sample estimator for the hard-SAT model when roughly $k\geq 6.45\log d$ (as $\beta\rightarrow 0$), and an impossibility result when roughly $k \leq \log d$. We say that a real $\hat{\beta}$ is a $\delta$-estimate of a real $\beta$ if $|\hat{\beta}-\beta|\leq \delta$.

\begin{theorem}
[Parameter Estimation of Hard-SAT]
\label{t:hard_sat}
Let $d,k\geq 3$ be integers. 
\begin{itemize}
    \item For $k \leq \log d - \log k + \Theta(1)$,  there exists a formula $\Phi=\Phi_{n,k,d}$ with a unique satisfying assignment, i.e., it is impossible to   estimate $\beta^\star$ from a sample $\sigma\sim\Pr_{\Phi,\beta^*}$.
    \item 
    Let $B,\epsilon>0$ be  arbitrary constants.   
    Then, for all
$k \geq \frac{6.45\log (d k) + \Theta(1)}{\log(1 + \emm^{-B})}$ and  $|\beta^\star| \leq B$ , there exists an 
     algorithm that, on input $\Phi = \Phi_{n,k,d}$ and a sample $\sigma\sim\Pr_{\Phi,\beta^*}$, outputs in linear time an $O(\frac{1}{\sqrt{n}})$-estimate\footnote{We use $O(1)$ to denote a constant depending only on $d,k,B,\epsilon$.} $\hat{\beta}_\sigma$ for $\beta^*$ with probability $\geq 1-\epsilon$ over the choice of $\sigma$, i.e., 
    $\Pr_{{\Phi,\beta^\star}}\left[\big| \hat{\beta}_\sigma - {\beta^\star}\big| =O(1/{\sqrt{n})} \right] \geq 1-\epsilon
    $.
\end{itemize}
\end{theorem}

At a technical level, the proof of this result poses various challenges compared to prior work in soft-constrained models. In particular, we develop a general technique, based on a coupling inspired by \cite{moitra2019approximate}, in order to show that the (negative) pseudo-likelihood objective is strongly convex. For the non-identifiability result above the LLL regime we refer to \Cref{theorem: hard-sat non-identifiable} and for the efficient algorithm we refer to  \Cref{sec:proofSAT} for a proof sketch.

\subsection{Results for $q$-colorings and general $H$-colorings}
We next consider the $q$-colorings model  and, more generally, the $H$-colorings model. These were previously studied in the context of single-sample learning by Bhattacharya and Ramanan \cite{bhattacharya2021parameter}, where they introduced a condition (the Rainbow Condition) and obtained an efficient estimator for learning the hardcore model (weighted independent sets). Here, we first investigate their condition as well as other standard conditions in the literature (such as Dobrushin's condition), and show that for general $H$-colorings models they are insufficient for single-sample learning. Then, building on the work of \cite{bhattacharya2021parameter}, we show that a variant of the Rainbow Condition, which we call the 2-Rainbow Condition, is in fact sufficient to imply single-sample learning. We then combine this result with an adaptation of the coupling technique that we developed for the case of $k$-SAT to conclude a linear-time estimator for $q$-colorings 
 and permissive models. 

\paragraph{Definitions.} In the $H$-coloring problem, we have a (fixed) graph $H$ on the set of vertices $[q]=\{1,\hdots,q\}$ which might contain self-loops. Given a graph $G=(V(G),E(G))$ with $n$ vertices,  
an \emph{$H$-coloring} of $G$ is an assignment $\sigma\in [q]^{n}$ of colors to $V(G)$ such that $(\sigma_u, \sigma_v)\in E(H)$ whenever $(u,v) \in E(G)$.  In words, two colors can be adjacent in a valid coloring of $G$ if and only if they are neighboring in $H$. The proper $q$-coloring model can be viewed as a special case of this framework with $H = K_q$ (the clique on $q$ vertices), i.e., no monochromatic edges are allowed in the coloring $\sigma$.  For the learning problem, we consider the space $\Omega^H_{G} \subseteq [q]^{V(G)}$  of all the $H$-colorings of the graph $G$, weighted according to the counts of the different colors using  a vector of parameters $\pmb{\beta} = (\beta_1,\ldots,\beta_{q}) \in \R^{q}$. More precisely, the MRF for the $H$-coloring model over $\Omega^H_{G}$ is given by 
\begin{equation}
\label{eq:h_coloring}
\Pr_{G, H, \pmb{\beta}}[\sigma] \propto 
\exp\bigg(\sum_{i=1}^{q} \beta_i c_i(\sigma) \bigg)\mathbf{1}\{\sigma \in \Omega^H_{G}\}, 
\end{equation}
where,  for $i\in [q]$ and $\sigma\in [q]^{V(G)}$, $c_i(\sigma)$  is the number of nodes in $G$ with color $i$.\footnote{\label{fn:betaqzero}Note that the sum of $c_i(\sigma) $ over $i\in [q]$ is $|V(G)|$ and so we can assume that $\beta_q = 0$ without loss of generality.
}
Our goal is, given a single sample $\sigma \sim \Pr_{G, H, \pmb \beta^\star}$, to efficiently compute a good estimate $ \wh{\pmb {\beta}}$ for $\pmb \beta^\star$ in $L_2$ norm.

Before proceeding, let us briefly comment on the structure of the colorings probability distribution. The vertex-based weightings we consider in this work (weighting each color individually) are analogous to the vertex-based weighting of the hard-core model (where, for an independent set $I$, a vertex $v$ contributes a factor of $\lambda$ if $v \in I$, and 1 otherwise) or the Ising model ``with external fields'' (where the weight of an Ising configuration also depends on the spin counts). Analogously weighted models have also been studied in various contexts, e.g., weighted $H$-colorings have been studied in approximate sampling and graph limits. We also remark that, at the level of generality we consider here, the vertex-based weighting of the distributions is not much different conceptually than having edge-based interactions as in the Ising model (e.g., having weights on the adjacency matrix of $H$); for example, on regular graphs vertex weights can also be viewed as edge weights. The important aspect of the distributions we consider, in either the vertex-based or edge-based weight formulation, is that one hard codes the zeros in the adjacency matrix of $H$ (so that adjacency of certain pairs of colors is ruled out).

\subsubsection{Negative results for $H$-colorings}
We study various conditions that exist in the literature for uniqueness and learnability of Gibbs distributions and show that they are not sufficient to guarantee (even) identifiability of the parameters for general $H$-coloring models.

For instance, we consider the Dobrushin and Dobrushin-Shloshman 
conditions \cite{dobruschin1968description}. These are well-studied conditions and are known to imply efficient sampling (via fast mixing of the heat bath Glauber
dynamics) 
\cite{hayes2006simple}.
Notably, \cite{blanca2018structure} shows that structure learning for $H$-colorings with multiple samples is possible when the DS condition holds true. We show that neither of these conditions suffices for learning from a single sample.
Also, we investigate the Rainbow Condition, introduced in \cite{bhattacharya2021parameter}, which is defined as follows:
suppose $\sigma \in \Omega_{G}^H$ is a proper $H$-coloring of $G$. Fix $r \in [q-1]$ and let us denote by $u_{r}(\sigma) \in [n]$ 
the number of vertices $v$ of $G$
that can be re-colored under $\sigma$ by color $r$ while keeping the colors of all the other vertices (i.e.,  $V\backslash \{v\}$) the same.
The Rainbow Condition roughly requires that  for all colors $r \in [q-1]$, $u_r(q)=\Omega(n)$ with probability  $1-o(1)$.
We show that Rainbow Condition also fails to imply one-sample learnability.
Our main negative result follows.

\begin{theorem}
[Negative Results for Parameter Estimation for $H$-colorings]
\label{theorem:negative results}
The following hold for any integer $n\geq 1$.
\begin{itemize}
    \item (Dobrushin fails) There exists a tuple $(G = G_n, H = H_q)$ and parameters $\pmb{\beta_1}, \pmb{\beta_2}\in \reals^q$ with $\|\pmb{\beta_1} - \pmb{\beta_2}\|_\infty = \Theta(1)$ such that
    the $H$-coloring model
    with these $2$ parameters
    satisfies Dobrushin's condition (see \Cref{def:dobrushin}) but $\dkl(P_{\pmb{\beta_1}} \parallel P_{\pmb{\beta_2}}) = O(n^2 \exp(-n))$.
    \item (Dobrushin-Shlosman fails) There exists a tuple $(G = G_n, H = H_q)$ and parameters $\pmb{\beta_1}, \pmb{\beta_2}\in \reals^q$ with $\|\pmb{\beta_1} - \pmb{\beta_2}\|_\infty = \Theta(1)$ such that
    the $H$-coloring model
    with these $2$ parameters
    satisfies Dobrushin-Shlosman's condition (see \Cref{def:dobrushin})
    but $\dkl(P_{\pmb{\beta_1}} \parallel P_{\pmb{\beta_2}}) = 0$.
    \item (Rainbow Condition fails) There exists a tuple $(G = G_n, H = H_q)$ such that
    the $H$-coloring model
    with any parameter $\pmb{\beta} \in \reals^q$
    satisfies the Rainbow Condition (see \Cref{cond:1})
    but for any $\pmb{\beta_1}, \pmb{\beta_2} \in \reals^q$, $\dkl(P_{\pmb{\beta_1}} \parallel P_{\pmb{\beta_2}}) = 0$.
\end{itemize}
\end{theorem}

\Cref{sec:insufficient} contains the proofs of the above collection of negative results about the one-sample learnability of $H$-colorings. Given the above results, it is interesting to investigate whether an efficient single-sample algorithm exists for $H$-colorings under mild conditions. This is the context of the upcoming section, where we show that a natural variant of the Rainbow Condition is sufficient for learning.

\subsubsection{Positive results for $H$-colorings}
The main goal of this section is to give a condition for general $H$-colorings that (i) will allow us to obtain an efficient single-sample learning algorithm, and (ii) is sufficiently mild to allow for a wide range of applications. Indeed, we use this to obtain a tight phase transition for the single-sample learnability of proper $q$-colorings (a result which was not known by prior work) and obtain efficient algorithms for the generic class of permissive models. 

Let us first introduce the aforementioned condition, which we call \emph{2-Rainbow Condition}. Suppose $\sigma \in \Omega_{G}^H$ is a proper $H$-coloring of an $n$-vertex graph $G$ and denote by $u_{rq}(\sigma)$ 
the number of vertices $v$ of $G$
that can be re-colored in $\sigma$ with both colors $r$ and $q$ (i.e., setting $\sigma(v)=r$ and $\sigma(v)=q$ while keeping the rest of $\sigma$ the same still yield proper $H$-colorings of $G$).
We say that $H$ satisfies the 2-Rainbow Condition (with constant $\delta>0$) on graphs $G$ of maximum degree $d$ if for all colors $r \in [q-1]$, $u_{rq}(\sigma)$ is lower bounded with probability close to 1 (as $n$ grows), i.e., 
\begin{equation}\label{eq:strong-rainbow}
\Pr_{G, H, \pmb \beta^*}
\big[u_{rq}(\sigma) \geq \delta n\mbox{ for all } r \in [q-1]\big] = 1-o(1).
\end{equation}
We note that this condition is a two-coloring refinement of the Rainbow Condition, introduced by \cite{bhattacharya2021parameter}, which as we showed in \Cref{theorem:negative results} does not suffice for learning $H$-colorings. 

The following result builds on related ideas to that of \Cref{t:hard_sat}. We sketch its proof in \Cref{sec:4g4gg4} and give the full proof in \Cref{sec:algo-h-coloring}. In the statement below, we use $O(1)$ to denote a constant depending only on $q,d,\delta,B,\epsilon$.
\begin{restatable}{theorem}{hcolpositive}
\textnormal{(Parameter Estimation for $H$-colorings).}
\label{theorem:positive result h-colorings}
Let $q,d \geq 2$ be integers and $B,\epsilon>0$ be reals. Consider an $H$-coloring model
    with parameter $\pmb{\beta}^* \in [-B,B]^q$
    that satisfies the 2-Rainbow Condition with constant $\delta>0$ (see \eqref{eq:strong-rainbow}) on max-degree $d$ graphs.    
    
    There exists a linear-time algorithm that, on input $G$ and a sample
    $\sigma \sim \Pr_{G, H, \pmb \beta^\star}$, outputs a $O(\frac{1}{\sqrt{n}})$-estimate $\wh{\pmb \beta}_\sigma$ for $\pmb \beta^\star$ w.p. $1-\eps$
    over $\sigma$, 
    i.e.,
    $
    \Pr_{G, H, \pmb{\beta^\star}}\left[ \big\|\wh{\pmb{\beta}}_\sigma - \pmb{\beta^\star}\big\|_2 = O(\frac{1}{\sqrt{n}}) \right] \geq 1-\eps.
    $
\end{restatable}

As applications of our general algorithmic result for $H$-colorings, we 
show that efficient single-sample learning of proper $q$-colorings is possible when $q > d+1$.
More to that, we can invoke one of the hard instances of \Cref{theorem:negative results} to show that this threshold is tight. Our result is hence a sharp phase transition for the problem of learning proper $q$-colorings with a single-sample: 
we show that if $q \geq d+2$, then estimation is possible with the standard minimax rate. In contrast, if $q \leq d+1$, we can construct a sequence of graphs $\{G_n\}_{n=1}^\infty$, where estimating the true $\pmb{\beta^\star}$ is impossible. 
The following Corollary formalizes these claims.
As mentioned earlier, we use $O(1)$ to denote a constant depending only on $q,d,\delta,B,\epsilon$.

\begin{restatable}{corollary}{pestimatecol}
\textnormal{(Parameter Estimation for Proper $q$-Colorings).}
\label{t:q_coloring}
Let $q,d \geq 2$ be integers and $B,\epsilon>0$.
\begin{itemize}
    \item 
    For $q \leq  d+1$,    there exists a sequence of graphs $\{G_n\}_{n}$ with max degree $d$ such that for any $\pmb{\beta_1}, \pmb{\beta_2} \in \R^{q}$, it holds that the KL
    divergence between $\Pr_{G_n, {\pmb{\beta_1}}}$ and $\Pr_{G_n, {\pmb{\beta_2}}}$ is 0, i.e., it is impossible to estimate the  parameter vector from a sample from the proper $q$-coloring model.
    \item For $q > d+1$ and $\bm \beta^*\in [-B,B]^q$, 
    there exists a linear-time algorithm that, on input $G $ with max degree $d$ and a sample
    $\sigma \sim \Pr_{G, \pmb \beta^\star}$, outputs an $O(\frac{1}{\sqrt{n}})$-estimate $\wh{\pmb \beta}_\sigma$ for $\pmb \beta^\star$ with probability $1-\eps$
    over $\sigma$, 
    i.e.,
    $\Pr_{G, \pmb{\beta^\star}}\left[ \|\wh{\pmb{\beta}}_\sigma - \pmb{\beta^\star}\|_2 = O(\frac{1}{\sqrt{n}}) \right] \geq 1-\eps.$
\end{itemize}
\end{restatable}

This result is in contrast to structure learning with multiple samples \cite{blanca2018structure},
where the problem is identifiable
in the tree uniqueness region (i.e., when $q > d$). 
The algorithmic part of \Cref{t:q_coloring} can be proven using our main result about parameter estimation for $H$-coloring distributions from one sample and so we refer to \Cref{appendix:h-coloring} for its proof. In particular, the case $q > d+1$ is handled in \Cref{sec:app-1} and the case $q \leq d+1$ appears in \Cref{t:q_coloring_2}.

Also,  we show how to apply our result to the general class of permissive models, i.e., the setting where there is an unconstrained color $r$ (connected in $H$ to all the  vertices, including itself). 
\begin{restatable}[Parameter Estimation for Permissive Models]{corollary}{corpermissive}
\label{cor:permissive}
Let $q,d \geq 2$ be integers,  $B, \eps > 0$ be reals and  $\pmb{\beta^\star} \in [-B,B]^q$. Let $H$ be a connected graph that has an unconstrained color.

There exists a linear-time algorithm that, on input $G $ with max degree $d$ and a sample $\sigma \sim \Pr_{G, H, \pmb \beta^\star}$, outputs an $O(\frac{1}{\sqrt{n}})$-estimate $\wh{\pmb \beta}_\sigma$ for $\pmb \beta^\star$ with probability $1-\eps$
    over the choice of $\sigma$,  i.e.,
    $
    \Pr_{G, H, \pmb{\beta^\star}}\left[\big\|\wh{\pmb{\beta}}_\sigma - \pmb{\beta^\star}\big\|_2 = O(\frac{1}{\sqrt{n}}) \right] \geq 1-\eps.
    $
\end{restatable}
The proofs of these two statements (for $q$-coloring with $q > d+1$ and permissive models) require careful coupling arguments that resemble that for the hard-SAT model. We sketch the main idea in \Cref{sec:sketchewr43r4} and give the details into \Cref{sec:app-1} and \Cref{sec:app-2}.

We further remark that, as we showed in \Cref{theorem:negative results}, the Rainbow Condition considered in \cite{bhattacharya2021parameter} is not sufficient to conclude these three results (\Cref{theorem:positive result h-colorings}, \Cref{t:q_coloring} and \ref{cor:permissive}).

\subsection{Related Work}\label{sec:relatedwork}

\paragraph{Parameter Estimation in Graphical Models.}
Single-sample statistical estimation
\cite{besag1974spatial,bresler2018optimal,dagan2020estimating}
has a long line of work: \cite{chatterjee2007estimation,bhattacharya2018inference,ghosal2020joint,dagan2020estimating} study the Ising model,
\cite{daskalakis2019regression}, \cite{mukherjee2021high} and \cite{kandiros2021statistical} study Linear and/or Logistic Regression, \cite{mukherjee2020estimation} studies the Tensor Ising model and \cite{diakonikolas2021outlier} designs robust algorithms for Ising models under Dobrushin's condition. In a related direction, \cite{dagan2019learning} studies statistical learning theory questions from one dependent sample. Finally, \cite{bhattacharya2021parameter} studies hard-constrained models, as we do. 

Estimation questions in graphical models with multiple samples have been also extensively studied: \cite{bresler2015efficiently} was the first work to obtain polynomial time algorithms for learning Ising models on bounded degree graphs. Improved guarantees where obtained in 
\cite{hamilton2017information,vuffray2016interaction,klivans2017learning}. Moreover, \cite{daskalakis2019testing,gheissari2018concentration} focus on the testing problem for Ising models.
For hard constrained models, \cite{blanca2018structure} studied structure learning for $H$-colorings with multiple samples  and \cite{bhattacharyya2021efficient} studies truncated sparse Graphical models. We remark that our hard-constrained models are special cases of truncated/censored distributions, by considering truncations of discrete product measures (as in \cite{fotakis2020efficient}) when the set has some special combinatorial structure. For instance, in our hard SAT model, we consider a truncated Boolean product distribution where the set contains all the satisfying assignments of a known CNF formula. Designing efficient algorithms for truncated \cite{daskalakis2018efficient,daskalakis2019computationally,nagarajan2020analysis,
kontonis2019efficient,fotakis2020efficient,fotakis2022perfect,de2023testing,he2023testing} and censored
\cite{ fotakis2021efficient,choo2022learning,plevrakis2021learning}
problems has recently gained attention from the computational learning theory community. 

Close to our hard-constrained setting, \cite{blanca2018structure} initiated the study of structure learning for $H$-colorings with multiple samples. For proper $q$-colorings, they showed that
in the tree uniqueness region ($q > d$) the problem is identifiable and there exists an efficient structure learning algorithm, while when $q \leq d$, the problem is not identifiable. 

The work of \cite{bhattacharya2021parameter} is the closest to our paper. \cite{bhattacharya2021parameter} also studied hard-constrained models. In particular, they provided an MPL-based algorithm for
the hardcore model (distribution over independent sets of a graph) 
and studied sufficient conditions for single-sample learnability for general $H$-colorings. For the latter model, they introduced the Rainbow Condition (see \Cref{cond:1}), a property that in general does not imply single-sample learnability (as we show in \Cref{theorem:negative results}).

\paragraph{Sampling from Hard-Constrained Distributions.}

The {hard-SAT} model (and its soft-SAT analogue) are standard across computer scientists \cite{achlioptas2003threshold, montanari2006counting,gao2008phase,galanis2021counting}, mathematicians \cite{talagrand2001high} and physicists \cite{monasson1996entropy}. Extensive work has been done in the area of efficiently finding and uniformly sampling satisfying assignments for SAT formulae; we refer to \cite{beck1991algorithmic,moser2010constructive,moitra2019approximate,chen2023algorithms,he2023improved,achlioptas2008algorithmic,bresler2022algorithmic,coja2010better,galanis2022fast,feng2021fast,guo2019uniform,jain2022towards} and the references therein. For the proper $q$-coloring problem, the sampling problem has been also studied quite well, see e.g., \cite{jerrum1995very,vigoda2000improved,molloy2002glauber,frieze2007survey,galanis2015inapproximability,chen2019improved,chen2021optimal,chen2021rapid,blanca2022mixing}. $H$-colorings are a fundamental concept in theoretical computer science, dating back to \cite{hell1990complexity}, and have been extensively studied in terms of counting and sampling \cite{dyer2000complexity,goldberg2002complexity,dyer2004relative,galanis2016approximately}.

\section{Overview: the Pseudo-Likelihood Objective}
\label{sec:mpl_general}\label{sec:overview}
Our parameter estimation algorithms are based on the Maximum Pseudo-Likelihood (MPL) estimator (introduced by Besag \cite{besag1974spatial}), which we now define. For a  vector $\sigma=(\sigma_1,\hdots,\sigma_n)$ and $i\in [n]$, we let $\sigma_{-i}$ denote the induced assignment on $[n]\backslash \{i\}$ and denote by $(\sigma_i=x,\sigma_{-i})$  the assignment that agrees with $\sigma$ on $[n]\backslash \{i\}$ and assigns the value $x$ to $i$. 

\begin{definition}
\label{def:mpl}
Given a discrete random vector 
$\sigma = (\sigma_1,...,\sigma_n)$
whose joint distribution is parametrized by a $q$-dimensional vector $\bm \beta$ lying in a region $\mathcal{R} \subseteq \reals^q$, the MPL estimate
of $\bm \beta$ is defined as
\begin{equation}\label{eq:mpl}
\wh{\bm \beta}_n 
=
\wh{\bm \beta}_n(\sigma) 
=
\argmin_{\bm \beta\in \mathcal{R}} \phi(\bm \beta,\sigma), \mbox { where } \phi(\bm \beta,\sigma)=-\log \prod_{i \in [n]}
f_i(\bm \beta, \sigma),
\end{equation}
and $f_i(\bm \beta, \sigma)$ is the conditional probability of $\sigma_i$ given 
$(\sigma_{-i})$, viewed as a function of $\bm \beta$.
\end{definition}

To argue about the consistency of the MPL estimate, the strategy devised by Chatterjee was to argue via the first/second derivatives. Namely, for $\phi_1=\nabla_{\bm{\beta}}\phi$ and $\phi_2=\nabla^2_{\bm{\beta}}\phi$, the key is to establish the following bounds (with $\mathbf{I}_q$ being the $q$-dimensional identity matrix):
\begin{equation}\label{eq:conditions}
\begin{aligned}
&\mbox{(i)}\ \E_{\sigma }
\left[ \left\| \phi_1(\bm{\beta^*};\sigma)\right\|_2^2 \right]
= O(n) \mbox{ where $\sigma\sim\Pr_{\Phi,\beta^*}$},\\
&\mbox{(ii)}
 \inf_{\bm \beta\in \mathcal{R}}\phi_2(\bm{\beta};\sigma) \succcurlyeq \Omega(n) \mathbf{I}_q \mbox{ with probability $1-o(1)$ over $\sigma\sim \Pr_{\Phi, \beta^*}$}.
 \end{aligned}
\end{equation}
Intuitively, the first-derivative condition in  (i) says that, for most $\sigma$, the (pseudo)likelihood is large for $\bm \beta$ close to $\bm \beta^*$, while the convexity condition in (ii) ensures that deviating substantially from ${\bm \beta}^*$ will cause a substantial loss in the (pseudo)likelihood probability for most $\sigma\sim \Pr_{\Phi, \beta^*}$.

Among the two conditions, it turns out that the most challenging to establish in our hard-constrained setting is (ii), the second derivative; previous analyses relied on the fact that the second derivative was essentially the statistic determining the distribution (e.g., in the case of $k$-SAT that would be $C(\sigma)$ determining the weight of $\sigma$). 

Prior work on soft-constrained models \cite{chatterjee2005concentration,bhattacharya2018inference,dagan2020estimating} has often relied on sophisticated concentration inequalities to establish strong convexity. 
In contrast, our proof is based on exploiting the combinatorial structure of the problem and hence we provide more direct arguments. 
In particular, one of the main contributions of this work is the development of a general coupling technique to handle the strong-convexity condition in (ii).  Namely, we construct a suitable coupling between $\sigma$ and a sequence of i.i.d. random variables, which then enables us to lower-bound the second derivative using standard Chernoff bounds. The details of constructing this coupling depend on the model (with $k$-SAT being more involved), but at a high-level the overall strategy is to greedily extract disjoint neighbourhoods to ensure independence/concentration; this extraction requires more effort for $k$-SAT but comes more naturally for our colorings results (with a fair bit of technical work in both cases).

\section{Proof of \Cref{t:hard_sat}}\label{sec:proofSAT}
To formulate the MPL estimate for $k$-SAT, consider a formula $\Phi$ with variable set $[n]=\{1,\hdots,n\}$ and a satisfying assignment $\sigma\in \Omega(\Phi)$. Recall that $C(\sigma)$ denotes the number of true variables under $\sigma$; for $j\in \{0,1\}$, we  extend this notation and let $C(x,\sigma_{-i})$ denote the number of true variables in the assignment obtained from $\sigma$ by setting the variable $i$ to $x$. We say that a variable $i\in [n]$ is flippable in $\sigma$ if flipping the value of $i$ under $\sigma$ while keeping the assignment of the remaining variables the same is still a satisfying assignment of $\Phi$. For a non-flippable $i$ we have $\Pr_{\Phi,\beta}[\sigma_i | \sigma_{-i}]=1$, whereas for a flippable $i$ we have $\Pr_{\Phi,\beta}[\sigma_i | \sigma_{-i}]=\tfrac{\emm^{\beta C(\sigma)}}{\emm^{\beta C(0,\sigma_{-i})}+\emm^{\beta C(1,\sigma_{-i})}}$. Hence, we can write the \emph{negative} log-pseudo-likelihood function for $k$-SAT, denoted by $\phi$ henceforth, as
\begin{equation}\label{eq:MPLkSAT}
\phi(\beta; \sigma) = 
-\sum_{i \in [n]} \log \Pr_{\Phi,\beta}[\sigma_i | \sigma_{-i}]
=
\sum_{i \in \calF(\sigma)} \left( \log \left(e^{\beta C(1, \sigma_{-i}) } + e^{\beta C(0, \sigma_{-i}) } \right) -\beta  C(\sigma) \right),
\end{equation}
where $\mathcal{F}(\sigma)$ denotes the set of flippable variables in $\sigma$.

We next show consistency of the MPL estimate using the two conditions in \eqref{eq:conditions}. The Lemma below establishes (the easier) condition (i). 
\begin{restatable}{lemma}
{lemmaupperbound}
\label{lemma:upper_bound}
For $\Phi = \Phi_{n,k,d}$ with $\Omega(\Phi)\neq \emptyset$ and $\beta \in \reals$, it holds that $\E\big[(\phi_1(\beta; \sigma))^2\big] \leq kd n$, where the expectation is over $\sigma\sim \Pr_{\Phi,\beta}$.
\end{restatable}
\begin{proof}[Proof Sketch.]
For $i\in [n]$ and $\sigma\in \Omega(\Phi)$, let $x_i$ denote a sample from $\Pr_{\Phi,\beta}[ \cdot \mid \sigma_{-i}]$ and $\sigma^{(i)}$ be the assignment obtained from $\sigma$ by setting the value of $i$ to $x_i$. From \eqref{eq:MPLkSAT}, we obtain that 
\begin{equation}\label{equation:hard-sat 1st derivative}
\mbox{$\phi_1(\beta;\sigma)
=
\sum_{i \in [n]} 
\Big(\E_{x_i}[C(x_i, \sigma_{-i})]
-C(\sigma)\Big),$}
\end{equation}
see \Cref{lemma:derivatives hard-sat} for details. Note that the terms in the sum corresponding to variables $i$ that are not flippable in $\sigma$ are equal to 0.

Let $I$ denote a uniform variable from $[n]$ and let $
f(\sigma) =
\E_{I}[C(\sigma) - C(\sigma^{(I)}) | \sigma]$. Using the exchangeable pairs technique of Chatterjee \cite{chatterjee2005concentration,chatterjee2007estimation}, we show in the Appendix that 
\begin{equation}\label{eq:main12321}
\tfrac{1}{n^2} 
\E_{\sigma } 
\left[
\left(
\phi_1(\beta; \sigma)
\right)^2
\right]
=
\tfrac{1}{2}\E_{\sigma, I}
\big[(f(\sigma) - f(\sigma^{(I)}))(C(\sigma) - C(\sigma^{(I)})\big].
\end{equation}
It is not hard to observe now that, for any $i\in [n]$,
$|C(\sigma) - C(\sigma^{(i)})| \leq 1$ and $|f(\sigma) - f(\sigma^{(i)})| \leq kd/n$ (for $\tau=\sigma^{(i)}$ and a variable $j\in [n]$, the quantities $S_j=C(\sigma)-C(\sigma^{(j)})$ and $T_{j}=C(\tau)-C(\tau^{(j)})$ can only differ if the variable $j$ is in some clause with variable $i$). This  yields that the right-hand side in \eqref{eq:main12321} is $\leq d k/n$,  implying the lemma.
\end{proof}

The proof for condition (ii) is more involved and is where we will need to use the regime of $k$ in \Cref{t:hard_sat}. We start by writing $\phi_2$, the second derivative of $\phi$. Analogously to \eqref{equation:hard-sat 1st derivative}, with $x_i\sim \Pr_{\Phi,\beta}(\cdot \mid \sigma_{-i})$, we show in \Cref{lemma:derivatives hard-sat} that
\begin{equation}\label{equation:hard-sat 2nd derivative}
\mbox{$\phi_2(\beta; \sigma) = \sum_{i \in [n]} \Var_{x_i}[C(x_i, \sigma_{-i})]$}.
\end{equation}
It will be useful to have a more combinatorial expression for $\phi_2$. Observe that if variable $i$ is not flippable in $\sigma$, then there is a unique way to set it if $\sigma_{-i}$ is fixed and hence $\Var_{x_i}[C(x_i, \sigma_{-i})]=0$. On the other hand, if $i$ is flippable in $\sigma$, we have $\Var_{x_i}[C(x_i, \sigma_{-i})]=\tfrac{\emm^{\beta}}{(\emm^\beta+1)^2}$. So,
\begin{equation}\label{equation:hard-sat 2nd derivativeb}
\phi_2(\beta; \sigma) = \frac{\emm^{\beta}}{(\emm^\beta+1)^2}\sum_{i \in [n]} e_i(\sigma),
\end{equation}
where $e_i(\sigma)$ is the indicator of the event that variable $i$ is flippable in $\sigma$ (i.e., \emph{every} clause where a literal for $i$ appears is satisfied by some variable $j\neq i$ in $\sigma$).
Condition (ii) in \eqref{eq:conditions} therefore amounts to showing that $\sum_{i\in [n]}e_i(\sigma) = \Omega(n)$ with high probability over the choice of $\sigma$.

\subsection{The coupling}\label{sec:couplingSATSAT} Our technique for lower-bounding the right-hand side of \eqref{equation:hard-sat 2nd derivative} (with high probability over the choice of $\sigma$) is inspired by so-called spatial-correlation properties studied in approximate sampling. For convenience, consider the graph $G_\Phi$ with vertex set $[n]$ where two variables $i$ and $j$ are connected by an edge when they appear together in a clause (either positively or negatively). In settings where sampling is possible, it is reasonable to expect that $\sigma_i$ and $\sigma_j$ are almost independent from each other when $i$ and $j$ are far away in $G_\Phi$, and hence potentially $e_i(\sigma)$ and $e_j(\sigma)$ will be almost independent as well (since these only depend on the one-step neighbourhood of the variables). 

There are a couple of caveats in applying this strategy for single-sample learning. First, in the approximate sampling setting, the standard techniques require typically even stronger correlation properties (rather than variable-to-variable correlations, one needs to control instead variable-to-sets correlations). Here, we want to avoid the stronger  properties to go beyond the sampling regime,  and approach instead the satisfiability threshold (this arises not only for the case of $k$-SAT but also for colorings and $H$-colorings).  Second, the most natural implementation of this approach does not quite work in our setting since instead of $\sigma_i$ and $\sigma_j$, we are actually interested in the events $e_i(\sigma)$ and $e_j(\sigma)$. Conditioning on $\sigma_i$ affects only one in $k$ literals for each of the $\leq d$ clauses that $i$ appears, so it is unlikely to matter much and one can hope to show the effect to die off based on what has not been revealed. Instead, the most natural way to condition on $e_i(\sigma)$ is to reveal the assignment of all variables that share a clause with $i$; unfortunately this opens up the possibility of a cascading effect where variables get biased from what has been revealed and that potentially reaches $j$.

Despite these caveats, it is still useful to start with the sampling idea of Moitra for $k$-SAT \cite{moitra2019approximate}. There, the key insight  was to approximate the marginals of a carefully selected subset of the variables , the so-called marked variables. The following is based on  the technique of \cite[Lemma 3.1]{moitra2019approximate} using the Lov\'asz Local Lemma (LLL), the proof is given in \Cref{sec:b333}. 

For $x,y\in (0,1)$, let $\mathrm{KL}(x \parallel y)= x \ln(x/y) + (1-x)\ln((1-x)/(1-y))$. Let $\lambda\in (0,1/2)$ be the solution to $\mathrm{KL}(\lambda \parallel 1/2)\log_2 e=\lambda/2$, so that $2/\lambda\approx 6.45$.
\begin{restatable}{lemma}
{markedunmarked}
\label{l:marked_unmarked} 
Suppose $k \geq (2/\lambda)\log (d k)+ \Theta(1)$. Consider a formula $\Phi=\Phi_{n,k,d}$. Then, there is a marking of the variables, 
$\mathcal{M}:[n]\to \{\text{marked, unmarked}\}$, such that each clause $c$ in the formula has at least $\lambda k$ marked and at least $\lambda k$ unmarked variables. 
\end{restatable}
A key feature of  the separation into marked and unmarked variables is that, no matter how we set the marked variables, the resulting formula is still satisfiable (as long as $k\gtrsim (2/\lambda)\log d$ thanks to the existence of the unmarked variables) using again the LLL; in fact, Moitra showed that for $\beta=0$ the marginal probability that a marked variable takes the value 1 is close to 1/2. We obtain an analogue of this property in the weighted setting, in fact we show the following Lemma (here, $\Phi'$ should be regarded as a subformula of $\Phi$ obtained by setting some marked variables, possibly all, in an arbitrary way).
\begin{restatable}{lemma}{flippableflip1}
\label{lemma:flippableone}
Fix a formula $\Phi'$ with maximum degree $d$ and set of variables corresponding to the set $[n]$, where every clause has at least $\lambda k$ and at most $k$ variables, where  $k \geq \tfrac{2\log (d k) + \Theta(1)}{\lambda\log(1 + \emm^{-\beta})}$.

For $i\in [n]$, let  $e_i(\sigma)$ be the event that $i$ is flippable in $\sigma\sim \Pr_{\Phi,\beta}$. Then $\Pr_{\Phi,\beta}[e_i(\sigma)=1] \geq 1/2$. 
\end{restatable}

At this stage, a natural approach to show the $\Omega(n)$ lower bound on $\phi_2$ would be to establish that for the marked variables $S$ it holds that $\sum_{i\in S} e_i(\sigma) = \Omega(n)$, since under any conditioning of the marked variables one can apply the lower bound in \Cref{lemma:flippableone}. Unfortunately, the  cascading issue explained earlier persists  since to determine $e_i(\sigma)$ we need to reveal the assignment of all neighbors of variable $i$ and this gets in the way of extracting some conditional independence among marked variables.

The observation now that simplifies things significantly is that we only need to lower bound $e_i(\sigma)$, so we can instead consider the indicator $g_i(\sigma)$ of the event that each clause $c$ where $i$ appears is satisfied by some \emph{marked} variable other than $i$. Note that $e_i(\sigma)$ is 1 if each clause is satisfied by some variable other than $i$ and hence trivially $e_i(\sigma)\geq g_i(\sigma)$. The benefit is twofold: (i) in the setting of \Cref{lemma:flippable} we can still show that $\Pr_{\Phi,\beta}[g_i(\sigma)= 1]\geq 1/2$, and more crucially (ii) to determine $g_i(\sigma)$ we only need to reveal marked variables in clauses where $i$ appears, allowing us to pick an appropriate subset of conditionally independent marked variables using a greedy approach.  With these ideas, we show the following Lemma in \Cref{sec:proofSATcoupling}.
\begin{restatable}{lemma}{couplingSATproof}\label{l:coupling}
Consider a formula $\Phi = \Phi_{n,k,d}$ with $k\geq \tfrac{2\log (d k) + \Theta(1)}{\lambda\log(1 + \emm^{-\beta})}$.

Then, there is a subset of the variables $R\subseteq [n]$ with $|R|\geq n/(2kd)^2$ such that, with probability $1-\emm^{-\Omega(n)}$ over the choice of $\sigma\sim \Pr_{\Phi,\beta}$, it holds that  $\sum_{i\in R} g_i(\sigma)\geq |R|/3$.
\end{restatable}

\subsection{Putting everything together}
\begin{proof}[Proof of \Cref{t:hard_sat}.] We  prove the statement for the existence of the estimator, the impossibility part is covered by~\Cref{theorem: hard-sat non-identifiable} in \Cref{sec:nonSATid}. 

Let $B,\eps>0$ be arbitrary constants and let $k\geq \tfrac{6.45\log (d k) + \Theta(1)}{\log(1 + \emm^{-B})}$ and $|\beta^*|\leq B$. Since since $2/\lambda< 6.45$, we have that both Lemmas~\ref{l:marked_unmarked} and~\ref{l:coupling} apply. From \Cref{lemma:upper_bound}, we have that 
\begin{equation}
\label{eq:consistency-1}
\E_{\sigma}
\left[\left(\phi_1(\beta^*; \sigma)\right)^2 \right] 
\leq dkn, \mbox{ where the expectation is over $\sigma\sim \Pr_{\Phi,\beta^*}$.}
\end{equation}
 Recall also from $\eqref{equation:hard-sat 2nd derivative}$ that $\phi_2(\beta;\sigma)=\tfrac{\emm^{\beta}}{(\emm^\beta+1)^2}\sum_{i \in [n]} e_i(\sigma)$ and, from \Cref{sec:couplingSATSAT}, $e_i(\sigma)\geq g_i(\sigma)$. Using the crude bound  $\tfrac{\emm^{\beta}}{(\emm^\beta+1)^2}\geq \tfrac{1}{\emm^{2B}(\emm^{2B}+1)^2} \triangleq \delta$ for $|\beta|\leq 2B$, we conclude from \Cref{l:coupling} that for all  $k\geq \tfrac{8\log d + \log k + \Theta(1)}{\log(1 + \emm^{-B})}$, it holds that
\begin{equation}\label{eq:t5g5g563f4guuu}
\Pr_{\Phi,\beta^*}\Big[\min_{\beta : |\beta|\leq 2B} \phi_2(\beta;\sigma)\leq \delta n/(4kd)^2\Big]=\emm^{-\Omega(n)}.
\end{equation}
Consider the event $\mathcal{A}=\{\sigma\in \Omega(\Phi):\,\big|\phi_1(\beta^\star; \sigma)\big|\leq 2dk\sqrt{n/\eps},\ 
\min_{\beta :|\beta| \leq  2B}  \phi_2(\beta; \sigma)\geq \delta n/(4kd)^2
\}$, and note that by a union bound using \eqref{eq:consistency-1} and \eqref{eq:t5g5g563f4guuu} it holds that $\Pr_{\Phi,\beta^*}[\mathcal{A}]\geq 1-\epsilon/2-o(1)\geq 1-\epsilon$ for all sufficiently large $n$.

Fix arbitrary $\sigma \in \mathcal{A}$ and let $\widehat{\beta}_{\sigma}$ be an MPL estimate corresponding to $\sigma$ so that $\phi_1(\widehat{\beta}_{\sigma},\sigma)=0$. For $t\in [0,1]$,  let $\beta_\sigma(t) = t\widehat{\beta}_\sigma+ (1-t) \beta^\star$ and 
$h_\sigma(t) = (\widehat{\beta}_\sigma - \beta^\star) \phi_1(\beta_\sigma(t);\sigma)$ for $t \in [0,1]$. Note that  $h_\sigma(1)=0$ and $h_\sigma'(t)=(\widehat{\beta}_\sigma - \beta^\star)^2 \phi_2(\beta_\sigma(t);\sigma)$, so 
\[-(\wh{\beta}_\sigma - \beta^\star)\phi_1(\beta^\star;\sigma)=h_n(1) - h_n(0)=(\wh{\beta}_\sigma - \beta^\star)^2 
\int_0^{1} \phi_2(\beta_\sigma(t);\sigma)\, \mathrm{d}t.\]
Note that $\phi_2(\beta;\sigma)\geq 0$ for every $\beta$ (cf. \eqref{equation:hard-sat 2nd derivative}), so we obtain
\begin{equation}\label{eq:ggg4546677}\big|\wh{\beta}_\sigma - \beta^\star\big| \big| \phi_1(\beta^\star;\sigma)\big|
\geq (\wh{\beta}_\sigma - \beta^\star)^2 
\bigg|\int_0^{\min\{1,B/|\wh{\beta}_\sigma - \beta^\star|\}} \phi_2(\beta_\sigma(t);\sigma) dt \bigg|.
\end{equation}
For $0\leq t\leq B/|\wh{\beta}_\sigma - \beta^\star|$, we have that $|\beta_\sigma(t)|\leq 2B$. Hence, since $\sigma \in \calA$, \eqref{eq:ggg4546677} gives that $
2dk\sqrt{n/\eps} \geq \frac{\delta n}{(4kd)^2} \min\{|\wh{\beta}_\sigma - \beta^\star|, B\}$. Therefore, for all sufficiently large $n$, we obtain that  $|\wh{\beta}_\sigma - \beta^\star| \leq M/\sqrt{n}$, where $M = \tfrac{(4kd)^3}{\delta\sqrt{\eps}}$, finishing the proof.
\end{proof}

\section{Proof sketch for $H$-colorings}\label{sec:4g4gg4}
In this section, we discuss the proof of \Cref{theorem:positive result h-colorings} and Corollaries~\ref{t:q_coloring} and \ref{cor:permissive}. 

We start with the main challenges for proving that the MPL estimate \eqref{eq:mpl} is consistent for $H$-coloring models that satisfy the $2$-Rainbow Condition \eqref{eq:strong-rainbow}. As per the discussion in \Cref{sec:mpl_general}, we focus on bounding the second derivative of the Pseudo-likelihood function $\phi$. We follow 
 the notation of \cite{bhattacharya2021parameter} where a similar framework was introduced for studying the Rainbow Condition.

Consider a graph $G$ with vertex set $[n]$, $H$ with vertex set $[q]$ and $\bm \beta^*=(\beta_1,\hdots, \beta_q)$. Assume w.l.o.g.\footnote{Since $\sum_{i=1}^q c_i(\sigma) = n$, the parameter vector $\bm{\beta}^\star = (\beta_1,\ldots,\beta_q)$ induces the same distribution as $(\beta_1 - \beta_q, \ldots, \beta_q - \beta_q)$.} that $\beta_q=0$, so   we are interested in estimating the $q-1$ parameters $\beta_1,\hdots,\beta_{q-1}$ using a sample $\sigma\sim \Pr_{G, H, \pmb{\beta}}$, cf. \eqref{eq:h_coloring}. For a node $u \in [n]$ and $s \in [q]$, define $Q_\sigma(u,s)$ to be $1$ if assigning the color $s$ to node $u$ and keeping the colors of other nodes according to $\sigma$ results in a proper coloring, and set $Q_\sigma(u,s) = 0$ otherwise. For $v\in [n]$ and $r \in [q]$, let us define
$
\theta_\sigma(v,\cdot)$ to be the distribution induced on $v$ conditioned on $\sigma_{-v}$, i.e., $\Pr_{G,H,\bm \beta}(\mid \sigma_{-u})$. Namely, for a color $r\in [q]$ we have $\theta_\sigma(v,r)=\frac{\emm^{\beta_r}Q_\sigma(v,r)}{\sum_{s=1}^q\emm^{\beta_s}Q_\sigma(v,s)}$.  
Then, the second derivative of $\phi$ at $\bm{\beta}$ can be computed as
\begin{align}\label{eq:hcol_second_der}
\phi_2(\pmb{\beta};\sigma) = \sum_{v=1}^n \lp(\diag(O_\sigma(v)) - O_\sigma(v)O_\sigma(v)^\top \rp)\,,
\end{align}
where $O_\sigma(v) = (\theta_\sigma(v,1),\ldots,\theta_\sigma(v,q-1))^\top$ is a $(q-1)$-dimensional vector. 
Our goal is to show that $\phi_2(\bm{\beta};\sigma)$ is strongly convex. The natural strategy would be to show that each of the matrices in that sum has lower bounded eigenvalues. However, in many interesting scenarios, these matrices are low rank. Thus, our strategy will be to instead argue that for every vector in $ \R^{q-1}$, there are $\Omega(n)$ matrices in the sum such that this vector does not belong in their nullspace. 
In particular, we establish the following inequality, whose proof can be found in \Cref{sec:algo-h-coloring}.

\begin{restatable}{lemma}{direction}\label{l:eigenvalue}
Let $v \in [n]$ and $\sigma \in \Omega^H_{G}$ and suppose that $Q_\sigma(v,q) = 1$. 
Let $\mathcal{I}(\sigma) =\{r \in [q-1]: Q_\sigma(v,r) = 1\}.$
Then, for any $\mathbf{x} = (x_1,\ldots,x_{q-1}) \in \R^{q-1}$, it holds that
\[
\mathbf{x}^\top \lp(\diag(O_\sigma(v)) - O_\sigma(v)O_\sigma(v)^\top \rp) \mathbf{x}  \geq \bigg(\frac{\min_{s \in [q]} \emm^{\beta_s}}{\sum_{s=1}^q \emm^{\beta_s}}\bigg)^2 \sum_{r \in \mathcal{I}(\sigma)} x_r^2.
\]
\end{restatable}
Intuitively, the lemma asserts that the matrix in the sum of \eqref{eq:hcol_second_der} preserves all coordinates corresponding to colors $r$ with $Q(v,r) = 1$, provided that $Q(v,q) = 1$. The 2-Rainbow Condition says that, with probability $1-o(1)$ over the choice of $\sigma$, for an arbitrary color $r\in [q-1]$, it holds that $u_{rq}(\sigma)=\Omega(n)$. This translates into $Q_\sigma(v,r)Q_\sigma(v,q) = 1$ for  $\Omega(n)$ vertices $v \in [n]$ yielding that $\bm{x}^\top \phi_2(\bm{\beta};\sigma)\bm{x} \geq K(\beta) x_r^2$ for all $\bm{x} \in \R^{q-1}$. Summing over $r \in [q-1]$ gives that $\phi_2(\bm{\beta};\sigma) \succcurlyeq \Omega(n) \mathbf{I}_{q-1}$. Since this holds  with probability $1-o(1)$ over the choice of $\sigma$, the proof for \Cref{theorem:positive result h-colorings} can be completed analogously to the proof of \Cref{t:hard_sat} given earlier, see \Cref{sec:algo-h-coloring} for details.

\subsection{Application to proper $q$-Colorings and permissive $H$-colorings}\label{sec:sketchewr43r4}
We now discuss how we can establish the 2-Rainbow Condition (see  \eqref{eq:strong-rainbow}) to obtain \Cref{t:q_coloring} and \Cref{cor:permissive}. We actually focus on the case of the proper $q$-coloring model, the idea for permissive $H$-coloring models is similar. 

Let $G=(V,E)$ be a graph of maximum degree $d$ with $V=[n]$, and assume that  $q \geq d+2$. For simplicity, let us fix $r \in [q-1]$ and denote for every node $i \in [n]$, $e_i^{rq}(\sigma) = Q_\sigma(i,r)Q_\sigma(i,q)$. In words, $e_i^{rq}(\sigma)$
is equal to 1 if assigning color $r$ or $q$ to node $i$ while keeping the other vertices colored as in $\sigma$ results in a proper coloring; otherwise is 0.
Note that $u_{rq}(\sigma) = \sum_{i=1}^n e_i^{rq}(\sigma)$.

Our goal is to show that $u_{rq}(\sigma)=\Omega(n)$ for any color $r\in [q-1]$ with high probability over $\sigma.$ Analogously to the hard-SAT case, such sums exhibit non-trivial correlations which are hard to analyze and concentration results are unfortunately not known for such models. Our solution will again be to construct a coupling, in order to ``extract'' as much independence as possible. 
In particular, let $R\subseteq V$ be a maximal subset of nodes that forms a 2-hop independent set of $G$, i.e., distinct $v,w \in R$ have disjoint neighborhoods. We can pick $R$ greedily; using that $G$ has maximum degree $d$, we have that $|R|\geq n/(d+1)^2$. 

The key to establish the 2-Rainbow Condition for $q\geq d+2$ is to show the following lemma, whose proof is given in \Cref{sec:app-1}.  The proof is quite delicate and requires careful use of the assumption $q\geq d+2$.
\begin{restatable}{lemma}{couplingcolorings}
\label{lemma:coupling-lower bound}
    Suppose $q\geq d+2$ and consider a color $r\in [q-1]$. Let $G$ be a graph of maximum degree $d$  and  $R=\{v_1,\hdots, v_t\}$ be a subset of vertices with mutually disjoint neighborhoods. Then, for any $k\in [t]$, 
     \[
     \Pr_{G, \pmb{\beta^\star}}[e_{v_k}^{rq}(\sigma) = 1|e_{v_1}^{rq}(\sigma) = x_1,\ldots,e_{v_{k-1}}^{rq}(\sigma) = x_{k-1} ] \geq \frac{1}{1+q^{d+1}\emm^{\|\pmb{\beta}^\star\|_1 d}},
     \]     
     for any values $x_1,\hdots,x_{k-1}$ of the random variables $e_{v_1}^{rq},\ldots,e_{v_{k-1}}^{rq}$ which are consistent with at least one proper $q$-coloring of $G$. 
\end{restatable}
Using the ``conditional independence'' property  of \Cref{lemma:coupling-lower bound}, analogously to  \Cref{l:coupling},  it is not hard to couple $\{e_{v}^{rq}\}_{v\in R}$  with independent i.i.d. random variables $\{\tau_v\}_{v\in R}$ with $\Pr[\tau_v=1]=\Omega(1)$. This yields that $u_{rq}(\sigma)=\Omega(n)$ with probability $1-\emm^{-\Omega(n)}$ over $\sigma$, establishing therefore the 2-Rainbow Condition (cf. \eqref{eq:strong-rainbow}). The proof of \Cref{t:q_coloring} is given in \Cref{sec:app-1}.

\subsection{Negative Results for $H$-colorings}
Dobrushin's and Dobrushin-Shlosman's (DS) condition are well-studied properties for graphical models \cite{dobruschin1968description} that are known to imply fast mixing and hence efficient sampling \cite{hayes2006simple}. Furthermore, the work of \cite{blanca2018structure} showed that, under the DS condition, there exists an efficient structure learning algorithm for $H$-colorings. Perhaps surprisingly as we show in \Cref{theorem:negative results} neither of these conditions suffices even for identifiability for single sample learning of $H$-colorings. In this Section, we discuss how to construct such hard instances.

The instance that showcases the insufficiency of Dobrushin's condition assumes that $G$ is a star graph with node $v$ in the middle and leaves $u_1,...,u_n$ and $H$ contains $q+1$ vertices so that color 1 is allowed to be adjacent to any other color and all other colors are only allowed to be adjacent to 1. Hence, the only valid $H$-colorings are $A = \{\sigma : \sigma_v = 1, \sigma_{u_i} > 1, \forall i \in [n]\}$ and
$B = \{\sigma : \sigma_v > 1, \sigma_{u_i} = 1, \forall i \in [n]\}$. 
Assume that the Hamiltonian counts the appearances of color 1.
It is not hard to show that event $B$ has exponentially small probability of occuring (by picking $\beta$ sufficiently small) and $\Var(c_1(\sigma) | \sigma \in A) = 0$. Then, the law of total variance implies that $\Var(c_1(\sigma))$ is of order $\exp(-n)$ for sufficiently small choices of $\beta$. Hence, one can then show that the KL divergence between two models with parameters $\beta_1$ and $\beta_2$ is exponentially small in
$n$, while $|\beta_1 - \beta_2|$ is lower bounded by a constant. This implies that parameter estimation is impossible. Moreover, one can verify that this instance satisfies Dobrushin's condition. For the formal statement, see \Cref{thm:dobrushin insufficient}.

In a similar manner we can design an instance that satisfies DS condition but where parameter estimation is not possible (\Cref{thm:ds insufficient}). We pick $G$ to be a cycle with $2n$ nodes $u_1,\ldots,u_{2n}$ and $H$ contains $q= 3$ vertices 
and has edge set $E = \{(1,2), (1,3)\}$. Let $A = \{u_1,u_3,\ldots\}, B = \{u_2,u_4,\ldots\}$ be the even and odd vertices. 
The valid $H$-colorings of $G$ will have $\sigma_A = 1$ and $\sigma_B \in \{2,3\}$ or $\sigma_B = 1$ and $\sigma_A \in \{2,3\}$. Suppose the Hamiltonian is of the form $\beta c_1(\sigma)$. Then, all valid $H$-colorings have exactly $n$ vertices of color $1$. Thus $\Var[c_1(\sigma)] = 0$ for all $\beta$, which means that learning $\beta$ is impossible. One can also verify that the DS condition holds true for this instance.

\bibliography{bib}

\appendix

\section{Remaining proofs for the Hard-SAT Model}
\label{sec:hard-sat}

\subsection{Proof of non-identifiability}\label{sec:nonSATid}
In this section, we establish the impossibility part of \Cref{t:hard_sat}. We construct instances $\Phi$ with  a single satisfying assignment, and therefore for any $\beta$, the probability distribution $\Pr_{\Phi,\beta}$ is a point mass on the unique satisfying assignment. The following lemma shows the construction.
\begin{lemma}
[Non-Identifiable Hard-SAT]
\label{theorem: hard-sat non-identifiable}
Let $k,d$ be integers with $d = (k+1)2^{k-1}$. There exists a sequence of formulae $(\Phi_{n,k,d})$  such that for each $n$ $\Phi_{n,k,d}$ has a unique satisfying assignment.
\end{lemma}
\begin{proof}
Let us consider $n$ variables $x_1,\ldots, x_n$.  For $i \in [n]$,
define $f(i) = i \mod (n+1)$ and let
\[
\Psi_i \triangleq (x_{f(i+1)} \lor \ldots \lor x_{f(i+k-1)})
\land (\overline{x}_{f(i+1)} \lor \ldots \lor x_{f(i+k-1)})
\ldots 
\land (\overline{x}_{f(i+1)} \lor \ldots \lor \overline{x}_{f(i+k-1)})\,.
\]
Namely, $\Psi_i$ contains $2^{k-1}$ clauses, corresponding to all possible combinations of the set of $k-1$ variables $x_{f(i+1)}, \ldots x_{f(i+k-1)}$. Let $\Psi_i'$ be the formula obtained by $\Psi_i$ by appending the literal $x_i$ to each clause. We consider the formula $\Phi=\land^n_{i=1} \Psi_i'\cdots \land \Psi_n'$ obtained by taking together all the $\Psi_i'$ for $i\in [n]$. Then the only assignment that satisfies $\Phi$ is by setting all $n$ variables equal to \texttt{True}. Moreover, it holds that any clause has size $k$ and the degree of any variable is $d = (k+1)2^{k-1}$. 
\end{proof}

\subsection{Derivatives of Pseudo-Likelihood Objective for Hard-SAT}
Our algorithm will rely on minimizing the negative pseudo-likelihood (NPL) objective $\phi(\cdot; \sigma)$, as defined in \eqref{eq:mpl}. In this section, we explicitly compute the first and second derivatives of this function.

\begin{lemma}
[Derivatives of NPL]
\label{lemma:derivatives hard-sat}
Consider a formula $\Phi = \Phi_{n,k,d}$ with $\Omega(\Phi)\neq \emptyset$, and let $\sigma\in \Omega(\Phi)$.
Let $\phi(\beta;\sigma_1)$ be the negative log-pseudo-likelihood of the Hard-SAT distribution defined from \eqref{eq:MPLkSAT}. For a variable $i\in [n]$, denote by $x_i$ the r.v. sampled from $\Pr(\cdot\mid \sigma_{-i})$. Then, 
\begin{equation}
    \phi_1(\beta; \sigma)= \sum_{i \in [n]} \left( \E_{x_i}[C(x_i, \sigma_{-i})] - C(\sigma) \right),\quad \mbox{ and }\quad
    \phi_2(\beta; \sigma) = 
    \sum_{i \in [n]} \Var_{x_i}[C(x_i, \sigma_{-i})].\label{equation:hard-sat-derivative}
    \end{equation}
\end{lemma}
\begin{proof}
Recall that  $\Pr_{\Phi,\beta}[\sigma] = \frac
    {\exp(\beta \sum_{j \in [m]} C(\sigma) ) 1\{\sigma \in \Omega\}}
    {Z(\beta)}$.  The first step is to study, for any $i \in [n]$, the conditional probability of the value of a single variable $\sigma_i$ given the rest $\sigma_{-i}$.

Recall that $\calF(\sigma) \subseteq [n]$ is the set of flippable variables under $\sigma$ and that
\begin{equation}\tag{\ref{eq:MPLkSAT}}
\phi(\beta; \sigma) = 
-\sum_{i \in [n]} \log \Pr_{\Phi,\beta}[\sigma_i | \sigma_{-i}]
=
\sum_{i \in \calF(\sigma)} \left( \log \left(e^{\beta C(1, \sigma_{-i}) } + e^{\beta C(0, \sigma_{-i}) } \right) -\beta  C(\sigma_i, \sigma_{-i}) \right)\,.
\end{equation}
Computing the derivative of $\phi$ with respect to $\beta$
\begin{align}
\phi_1 (\beta; \sigma)
& =
\sum_{i \in \calF(\sigma)}
\left( 
\frac
{C(1, \sigma_{-i}) e^{\beta S(1, \sigma_{-i})} 
+ 
C(0, \sigma_{-i}) e^{\beta S(0, \sigma_{-i})})}
{e^{\beta C(1, \sigma_{-i})} + e^{\beta C(0, \sigma_{-i})}}
-C(\sigma)
\right)\label{eq:54g45g45}\\
& =
\sum_{i \in [n]} 
\left(
\E_{x_i}[C(x_i, \sigma_{-i})]
-C(\sigma)
\right),\nonumber
\end{align}
which shows the expression for $\phi_1(\beta;\sigma)$ given in \eqref{equation:hard-sat-derivative}. Moreover, differentiating \eqref{eq:54g45g45} with respect to $\beta$ gives the expression for $\phi_2(\beta;\sigma)$ given in \eqref{equation:hard-sat-derivative}.
\end{proof}


\subsection{Expectation of 1st Hard-SAT Derivative}\label{sec:SATexpectation}
We next give the full proof of \Cref{lemma:upper_bound}, which we restate here for convenience.
\lemmaupperbound*
\begin{proof}
Let us define
\[
F(\sigma, \sigma') = C(\sigma) - C(\sigma')\,,~~~\sigma, \sigma' \in \{0,1\}^n\,,
\]
where $C(\sigma)$ is the number of variables set to true in $\sigma$. Consider an assignment $\sigma \in \{0,1\}^n$ and choose a coordinate $I \in [n]$ uniformly at random and replace the $I$-th coordinate of $\sigma$ by a random sample drawn from the conditional distribution $\Pr_{\Phi,\beta}(\cdot\mid \sigma_{-i})$. Let $\sigma^{(I)}$ be the resulting assignment. 
Moreover, consider the function $f$ defined by
\[
f(\sigma) = \E_{I} [F(\sigma, \sigma^{(I)}) | \sigma]
=
\E_{I}[C(\sigma) - C(\sigma^{(I)}) | \sigma],
\]
Denoting by $x_i\sim \Pr_{\Phi,\beta}(\cdot \mid \sigma_{-i})$, we can rewrite this as
\begin{equation}\label{eq:trbg34g3}
f(\sigma) = \frac{1}{n} \sum_{i \in [n]}
\big(C(\sigma) - \E_{x_i}[C(x_i, \sigma_{-i})|\sigma_{-i}]\big)= -\frac{1}{n} \phi_1(\beta;\sigma),
\end{equation}
where the second equality follows from the expression for $\phi_1$ given in  \eqref{equation:hard-sat-derivative}. 
Notice that
\[
f(\sigma) = \frac{1}{n} \sum_{i \in [n]}
\sigma_i - \E_{x_i}[x_i]\,,
\]
We observe that $(\sigma, \sigma^{(I)})$ is an exchangeable pair, i.e., we have that
\[
\E_{\sigma}[f(\sigma)^2]
=
\E_{\sigma}[f(\sigma) f(\sigma)]
=
\E_{\sigma, I}[f(\sigma) F(\sigma, \sigma^{(I)})]
=
\E_{\sigma, I}[f(\sigma^{(I)}) F(\sigma^{(I)}, \sigma)]\,.
\]
Since $F$ is anti-symmetric, i.e., $F(\sigma, \sigma') = -F(\sigma', \sigma)$, we have that
\begin{equation}\label{eq:g4tg45t532}
\E_{\sigma}[f(\sigma)^2]
=
\frac{1}{2} \E_{\sigma, I}
[(f(\sigma) - f(\sigma^{(I)}))F(\sigma, \sigma^{(I)})]\,.
\end{equation}
Combining \eqref{eq:trbg34g3} and \eqref{eq:g4tg45t532}, we obtain that
\begin{equation*}\tag{\ref{eq:main12321}}
\frac{1}{n^2} 
\E_{\sigma} 
\left[
\left(\phi_1(\beta; \sigma)
\right)^2
\right]
=
\frac{1}{2}\E_{\sigma, I}
[(f(\sigma) - f(\sigma^{(I)}))(C(\sigma) - C(\sigma^{(I)})].
\end{equation*}
Observe now that, for any $i\in [n]$,
$|C(\sigma) - C(\sigma^{(i)})| \leq 1$ and $|f(\sigma) - f(\sigma^{(i)})| \leq kd/n$. The former inequality is immediate, while for the latter it suffices
to note that, for $\tau=\sigma^{(i)}$ and a variable $j\in [n]$, the quantities $S_j=C(\sigma)-C(\sigma^{(j)})$ and $T_{j}=C(\tau)-C(\tau^{(j)})$ can only differ if the variable $j$ is in some clause with variable $i$). This  yields that the right-hand side is $O(1/n)$,  implying the Lemma.
\end{proof}

\subsection{Marking}\label{sec:b333}
We next give the proof of \Cref{l:marked_unmarked}.
\markedunmarked*
\begin{proof}
We choose each variable to be marked or unmarked with probability $1/2$ independently. Consider the $m$ bad events, one for each clause $c$, that $c$ does not have enough marked or unmarked variables. Then, by the Chernoff-Hoeffding Theorem, we have that
\[
\Pr[c \text{ is bad}] \leq 2 e^{-k \mathrm{KL}(\lambda \parallel 1/2)} = 2^{1-k \mathrm{KL}(\lambda \parallel 1/2) \log_2 e},
\]
where recall that $\mathrm{KL}(p \parallel q)= p\ln \frac{p}{q} + (1-p)\ln \frac{1-p}{1-q}$. By the Lov\'asz Local Lemma, if 
$$
4e(kd+1) \leq 2^{k \mathrm{KL}(\lambda \parallel 1/2) \log_2 e},
$$
then there exists an assignment of marked and unmarked variables that satisfies the constraints. By taking logs in the above inequality and since $\lambda$ satisfies $\mathrm{KL}(\lambda \parallel 1/2) \log_2 e=\lambda/2$, the result follows. 
\end{proof}

\subsection{Probability lower bounds using the LLL}
In this section, we prove \Cref{lemma:flippableone}; in fact, we will show a slight strengthening of it, see \Cref{lemma:flippable} below, which will be more suitable for the upcoming proof of \Cref{l:coupling} in \Cref{sec:proofSATcoupling}. 

We begin by stating  the following version of the LLL for $k$-SAT formulas. 
\begin{lemma}
[\cite{guo2019uniform}]
\label{lemma:LLL}
Let $\Phi$ be a SAT formula with set of variables $V$ and set of  clauses $\mathcal{C}$. For $\beta\in \mathbb{R}$, let $\mu_\beta$ be  the product distribution  on $\{0,1\}^V$ where each variable of $\Phi$ is set to 1 independently with probability $\tfrac{\emm^\beta}{1+\emm^\beta}$.

Suppose there is a function $x : \mathcal{C} \to (0,1)$ such that for all $c \in \mathcal{C}$ we have that
\[
\Pr_{\mu_\beta}[c \textnormal{ is unsatisfied}]
\leq x(c) \prod_{b \in \Gamma(c)}(1-x(b))\,.
\]
Then, $\Phi$ has at least one satisfying assignment. Moreover, suppose that 
$E$ is an event that is completely determined by the assignment of some set $S$ of the variables, and let $\Gamma(E)$ denote all the clauses that contain any of the
variables in S. Then, $
\Pr_{\Phi,\beta}[E] \leq \Pr_{\mu_\beta}[E] 
\prod_{b \in \Gamma(E)}(1-x(b))^{-1}$.
\end{lemma}
The above Lemma follows from Lemma 27 of \cite{guo2019uniform} by setting $S = m$ and by picking the bad event to be the event that clause $i$ is unsatisfied.  We are now ready to prove the following strengthening of \Cref{lemma:flippableone}.

\begin{restatable}{lemma}{flippableflip}
\label{lemma:flippable}
Fix a formula $\Phi$ with max degree $d$ and set of variables corresponding to the set $[n]$, where every clause has at least $\lambda k$ and at most $k$ variables. Suppose that $k \geq \tfrac{2\log (d k) + \Theta(1)}{\lambda\log(1 + \emm^{-\beta})}$ and  consider a marking $\mathcal{M}:[n]\mapsto\{\text{marked, unmarked}\}$ of the variables.

For $i\in [n]$ and $\sigma\sim\Pr_{\Phi,\beta}$, suppose that all clauses that contain $x_i$ have at least $\lambda k$ marked and $\lambda k$ unmarked variables. Let  $g_i(\sigma)$ be the indicator of the event that every clause $c$ containing $i$ is satisfied in $\sigma$ by a marked variable other than $i$. Then $\Pr_{\Phi,\beta}[g_i(\sigma)=1] \geq 1/2$. 
\end{restatable}
\begin{proof}[Proof of Lemmas~\ref{lemma:flippableone} and~\ref{lemma:flippable}.]
We are going to use \Cref{lemma:LLL} in order to upper bound $\Pr_{\Phi,\beta}[g_i(\sigma)=0]$. For every clause $c$, we set $x(c) = \frac{1}{D + 1}$ with $D = d^2 k$  and check the condition of the asymmetric LLL, i.e., that $\Pr_{\mu_\beta}[c \textnormal{ is unsatisfied}]
\leq x(c) \prod_{b \in \Gamma(c)}(1-x(b))$, where $\mu_\beta$ is  the product distribution where each variable is set to 1 independently with probability $\tfrac{\emm^\beta}{1+\emm^\beta}$.
For $k \geq \frac{2\log (d k) + \Theta(1)}{\lambda\log(1+\emm^{-\beta})}$, we have 
\begin{equation}\label{eq:g5g5uug}
 \emm (d^2 k +1) \leq (1+\emm^{-\beta})^{\lambda k} \mbox{ and hence } \bigg(\frac{\emm^\beta}{1+\emm^\beta} \bigg)^{\lambda k}
\leq \frac{1}{D+1}\left( \frac{D}{1+D}\right)^{D},
\end{equation}
which establishes the LLL condition since $|\Gamma(c)|\leq D=d^2k$.

We can apply now \Cref{lemma:LLL}. Namely, we first have the bound 
\[\Pr_{\mu_\beta}\big[g_i(\sigma)=0\big]
=\Pr_{\mu_\beta}[\exists ~\text{clause containing $i$ not satisfied by a marked variable}]\leq
d \cdot \left(\frac{\emm^\beta}{1+\emm^\beta}\right)^{\lambda k-1}.\]
Now since the event $g_i(\sigma)=0$ is completely determined by $\leq dk$ variables (those that share a variables with $i$),  and we have at most $D= d^2 k$ clauses that contain such a variable, we get that
\[
\Pr_{\Phi,\beta}\big[g_i(\sigma)=0\big]
\leq 
\Pr_{\mu_\beta}\big[g_i(\sigma)=0\big]
\left(1 - \frac{1}{D+1} \right)^{-D}
\leq 
\emm d \left( \frac{\emm^\beta}{1 + \emm^\beta} \right)^{\lambda k-1}, 
\]
using that $(1 - \tfrac{1}{D+1})^{-D} \leq \emm$. For $k  \geq \frac{\log d+\Theta(1)}{\lambda \log(1+\emm^{-\beta})}$,  we therefore have that $\Pr_{\Phi,\beta}[g_i(\sigma)=0]\leq 1/2$,   implying the desired result.
\end{proof}


\subsection{Proofs for the coupling}\label{sec:proofSATcoupling}
We will use the following lemma.
\begin{lemma}[see, e.g., {\cite[Section 23.9]{friezebook}}]\label{lem:dominance}
Suppose that $\{Y_i\}_{i\in [n]}$ are independent random variables and that $\{X_i\}_{i\in [n]}$ are random variables so that for any real $t$ and $i\in [n]$, it holds that $\Pr[X_i\geq t\mid X_1,\hdots X_{i-1}]\geq \Pr[Y_i\geq t]$. Then, for any real $t$,
\[\Pr[X_1+\cdots+X_n\geq t]\geq \Pr[Y_1+\cdots+Y_n\geq t].\]
\end{lemma}
We are now ready to prove \Cref{l:coupling}, which we restate here for convenience.
\couplingSATproof*
\begin{proof}
Consider a marking of the variables
$\mathcal{M}: [n]\rightarrow \{\text{marked, unmarked}\}$
such that every clause that
at least $\lambda k$ marked and 
$\lambda k$ unmarked variables whose existence is guaranteed by \Cref{l:marked_unmarked}. 
Consider the graph $G$ whose vertices are the variables of $\Phi$, where two vertices are adjacent if the corresponding variables appear in the same clause of $\Phi$. Note that $G$ has maximum degree $dk$.  Let $R=\{i_1,\hdots,i_r\}$ be a maximal subset of the vertices of $G$ whose neighbourhoods are mutually disjoint;  note that the greedy algorithm to pick such a subset $R$ (by picking an arbitrary vertex into $R$, deleting its 2-hop neighbors, and recursing) implies that $r\geq n/(1+kd)^2\geq n/(2kd)^2$. 

Let $\sigma \sim \Pr_{\Phi,\beta}$. We will show that for $t=1,\hdots,r$, it holds that
\begin{equation}\label{eq:ttttt3}
\Pr_{\Phi,\beta}[g_{i_{t}}(\sigma)=1|g_{i_1}(\sigma),\ldots,g_{i_{t-1}}(\sigma)] \geq 1/2.
\end{equation}
By \Cref{lem:dominance} and standard Chernoff bounds, we obtain that
$\sum_{i\in R} g_i(\sigma)\geq |R|/3$ with probability $1-\emm^{-\Omega(n)}$ over the choice of $\sigma\sim \Pr_{\Phi,\beta}$.

To prove \eqref{eq:ttttt3}, consider arbitrary $t=1,\hdots,r-1$. For $i\in [n]$, let $\mathcal{M}_i$ denote the marked variables that are neighbors of $i$,  and let $
U_t = \cup^{t-1}_{j=1}\mathcal{M}_{i_j}$
be the set of marked variables that are neighbors of the variables $i_1,\ldots,i_{t-1}$. Let us condition on  $\sigma_{U_{t-1}}=\tau$, i.e., that  the restriction of the assignment $\sigma$ on $U_{t-1}$ is given by $\tau$. Let $\Phi_\tau$ be the formula obtained by $\Phi$ by setting the variables in $U_{t-1}$ according to $\tau$, i.e., removing literals that are not satisfied by the partial assignment $\tau$ and removing clauses that are already satisfied by  $\tau$. Note that every every clause of $\Phi$ has at least $\lambda^* k$ unmarked variables by the marking $\mathcal{M}$, and $\tau$ only sets marked variables (the set $U_t$), so every clause in $\Phi_\tau$ has at least $\lambda^* k$ variables. 

By the law of total probability and the  fact that $g_{i_1},\ldots,g_{i_{t-1}}$ are deterministic functions of $\sigma_{U_t}$,  we have that
\begin{align}
\Pr[g_{i_{t}}=1|g_{i_1},\ldots,g_{i_{t-1}}] 
&=  \sum_{\tau \in \{0,1\}^{U_t}} 
\Pr[\sigma_{U_t} = \tau|g_{i_1},\ldots,g_{i_{t-1}}]\Pr[g_{i_{t}}=1|\sigma_{U_t} = \tau].\label{eq:4t44ft45345}
\end{align}
Conditioning on $\sigma_{U_t} = \tau$ results in a formula with at least $\lambda^* k$ variables in every clause. We have also picked $i_{t}$ so that it does not have any common neighbors with $i_1,\ldots,i_{t-1}$. Thus, the marked neighbors $\mathcal{M}_{i_{t}}$ of $i_{t-1}$ have not been set by the conditioning $\sigma_{U_t} = \tau$. Hence, by \Cref{lemma:flippable}, we have that 
\[\Pr_{\Phi,\beta}[g_{i_{k+1}}(\sigma)=1|\sigma_{U_k}=\tau]\geq 1/2 \mbox{ for all $\tau\in \{0,1\}^{U_t}$}.\]
Plugging this bound into \eqref{eq:4t44ft45345} gives \eqref{eq:ttttt3} as wanted, completing therefore the proof of the Lemma.
\end{proof}

\section{Insufficient conditions for parameter estimation for $H$-colorings}
\label{sec:insufficient}
In this section, we show that various natural well-studied conditions in the graphical models literature do not suffice for parameter estimation for  $H$-colorings. 

We will use the following lemma, which is standard for exponential families \cite{busa2019optimal,kalavasis2022learning}.

\begin{lemma}
\label{lemma:exp family kl}
Let
$P_{\pmb \beta}$ denote a distribution with parameter $\pmb \beta\in \mathbb{R}^q$ and density proportional to $e^{\< \pmb \beta, \pmb H(\sigma)\>}$ for $\sigma\in \Omega$,
where $\pmb H : \Omega \to \reals^q$.
Then, for $\pmb \beta, \pmb\gamma \in \reals^{q}$, there is $\xi\in (0,1)$ such that
\[\dkl(P_{\pmb \beta} \parallel P_{\pmb \gamma}) = \frac{1}{2} \Var_{\sigma\sim P_{\pmb\beta + \xi( \pmb\gamma - \pmb\beta)}} \lp[ \< \pmb \gamma - \pmb \beta, \pmb H(\sigma) \> \rp].\]
\end{lemma}
\begin{proof}
Consider the free energy function $F : \reals^q \to \reals$ with
$F(\pmb \beta) = \log \left( \sum_{\sigma \in \Omega} e^{\< \pmb \beta, \pmb H(\sigma)\>} \right)$.
Define the function $g:[0,1]\to \reals$ with 
$g(t) = F(\pmb \beta + t(\pmb \gamma- \pmb \beta))$. Then, we have that
\begin{align*}
&g'(t) = \frac{ \sum_{\sigma \in \Omega} \< \pmb \gamma - \pmb \beta, \pmb H(\sigma) \> e^{\<\pmb \beta + t (\pmb \gamma - \pmb \beta), \pmb H(\sigma) \>}}{ \sum_{\sigma \in \Omega} e^{\<\pmb \beta + t (\pmb \gamma - \pmb \beta), \pmb H(\sigma) \>} }  = \E_{\sigma \sim P_{\pmb \beta + t(\pmb \gamma - \pmb \beta)}} \lp[ \< \pmb \gamma - \pmb \beta, \pmb H(\sigma) \>\rp]\,.\\
&g''(t) = \Var_{\sigma\sim P_{\pmb \beta + t(\pmb \gamma - \pmb \beta)}} \lp[ \< \pmb \gamma - \pmb \beta, \pmb H(\sigma) \>\rp]\,.
\end{align*}
By Taylor's theorem, there exists $\xi\in (0,1)$ such that $g(1) - g(0) - g'(0)= \frac{g''(\xi)}{2}$. Therefore
\begin{align*}
\dkl(P_{\pmb \beta}  \parallel P_{\pmb \gamma} ) &= \E_{\sigma \sim P_{\pmb \beta}} \ln \frac{P_{\pmb \beta}(\sigma)}{P_{\pmb \gamma}(\sigma)}= F(\pmb \gamma) - F(\pmb \beta) - \E_{\sigma\sim \pmb \beta} \lp[ \< \pmb \gamma - \pmb \beta, \pmb H(\sigma) \>\rp]\\
&= g(1) - g(0) - g'(0)= \frac{g''(\xi)}{2}= \frac{1}{2}\Var_{\sigma\sim P_{\pmb \beta + \xi(\pmb \gamma - \pmb \beta)}} \lp[ \< \pmb \gamma - \pmb \beta, \pmb H(\sigma) \>\rp],\qedhere
\end{align*}
\end{proof}


We next recall the standard Dobrushin and Dobrushin-Shlosman (DS) conditions \cite{dobruschin1968description}.
\begin{definition}
[Dobrushin and Dobrushin-Shlosman Conditions]
\label{def:dobrushin}
Let $H$ be an arbitrary constraint graph and let $G = (V, E)$ be an $H$-colorable
graph.
For $u,w \in V$, let 
\begin{align*}
T_w& = \{ (\tau, \tau_w) : \tau, \tau_w \in [q]^{V} \mbox{ such that } \tau(z) = \tau_w(z)\mbox{ for all } z \neq w \},\\
R_{vw}&= \max_{(\tau, \tau_w) \in S_w \subseteq T_w} \big\| \pi_v(\cdot \mid \tau(\calN(v)) ), \pi_v(\cdot \mid \tau_w(\calN(v)))\big\|_{\mathrm{TV}},
\end{align*}
where $S_w$ is the set of configurations such that the conditional distributions $\pi_v(\cdot | \tau(\calN(v)) )$ and  $\pi_v(\cdot | \tau_w(\calN(v)))$  are well-defined. Then,
\begin{itemize}
    \item $\pi_G$ satisfies Dobrushin's uniqueness condition if
    $\max_{v \in V} \sum_{w \in \calN(v)} R_{v w} < 1$.
    \item $\pi_G$ satisfies Dobrushin-Shlosman's uniqueness condition if $\max_{v \in V} \sum_{w \in V} R_{w v} < 1$.
\end{itemize}

\end{definition}

First, we show that Dobrushin's condition is not sufficient for learning.
\begin{theorem}
[Dobrushin does not suffice]
\label{thm:dobrushin insufficient}
For any integers $n,q\geq 3$,
there exist graphs $(G = G_n, H = H_q)$ and parameters $\beta_1 < \beta_2$ where $\beta_2 - \beta_1$ is a sufficiently small constant 
such that
the H-coloring model
$\Pr_{G, H, \beta}$ satisfies Dobrushin's condition
for any $\beta \in (\beta_1, \beta_2)$
but for any $v \in V(G_n)$ and any $i \in [n]$, it holds that
\[
\Var_{
\sigma \sim {G, H, \beta}}(c_1(\sigma)) = O(n^2 \exp(-n)) \mbox{ for $\beta \in (\beta_1, \beta_2)$}\,,
\]
and, consequently 
\[
\dkl(P_{\beta_1} \parallel P_{\beta_2}) = O(n^2\exp(-n))\,.
\]
\end{theorem}
\begin{proof}
Let $G_n$ be a star graph with a node $v$ in the middle and leaves $u_1,\ldots, u_n$. The set of colors is $[q+1]$ and the constraint graph $H$ has edge set $E = \{(1,i):i>1\}$. 
In other words, color $1$ is allowed to be adjacent to any other color, and all other colors are only allowed to be adjacent to $1$. 
It is clear that the set of valid colorings is $A \cup B$ where 
\[
A = \left\{\sigma_{v} = 1 , \sigma_{u_i} > 1,~ \forall i \in [n]\right\}, \quad~~ B =
\left\{\sigma_v > 1, \sigma_{u_i} = 1,~ \forall i \in [n]\right\}\,.
\]
For simplicity, assume that $\beta$ is a sufficiently small constant (to be specified later). Then, the event $B$ has exponentially small (in $n$) probability of occurring. In particular, under $\sigma \in B$, we
have that $c_1(\sigma) = n$ and $\Pr[\sigma \in B] = \frac{q \exp(n \beta)}{q^n\exp(\beta) + q \exp(n \beta)}$.
Since we would like that
$q^{n-1} \exp(-n\beta) > \exp(cn)$, for some constant $c>0$,
we have to pick $\beta$ a sufficiently small constant.
Thus, we can focus without loss of generality on the event $A$. Since $c_1(\sigma) = 1$ for all $\sigma \in A$, it follows that $\Var(c_1(\sigma)|\sigma \in A) = 0$.
By the law of total variance, this implies that $\Var(c_1(\sigma))$ will be exponentially small for small enough $\beta$. This means that learning $\beta$ is impossible in this regime. 

To be more specific, we have that
\[
\Var(c_1(\sigma))
= 
\E_{\sigma}[\Var_{\sigma_v}(c_1(\sigma) | \sigma_{-v})]
+ 
\Var_{\sigma}(\E_{\sigma_v}[c_1(\sigma) | \sigma_{-v}])\,.
\]
By the above discussion, the first term will be 0. For the second one, note that
$\E_{\sigma_v}[c_1(\sigma) | \sigma_{-v}]$ is a random variable that is equal to $n$ with probability at most $\exp(-c n)$ and $1$ with probability at least $1-\exp(-c n)$. This implies that the second term of the formula is upper bounded by $C n^2 \exp(-c n)$, by picking $c$ sufficiently large and some positive constant $C$.
Since the above calculations hold for any parameter $\beta$ that is sufficiently small, we can invoke \Cref{lemma:exp family kl}
and get that 
\[
\dkl(P_{\beta_1} \parallel P_{\beta_2}) = O(n^2\exp(-n))\\,
\]
while $|\beta_1 - \beta_2|$ is lower bounded by a constant. This implies that retrieving the true parameter is not possible.

Now, let us compute the Dobrushin matrix of this distribution. The only non-zero entries will be $R_{vu_i}$ and $R_{u_iv}$. Let us start with $R_{vu_i}$ first. Notice that the value of some $u_j$ for $j \neq i$ already determines the conditional distribution of $\sigma_u$. In particular, if $u_j > 1$ then $\sigma_u = 1$ and if $u_j = 1$ then $\sigma_u$ is uniform in $\{2,\ldots,q+1\}$. Thus, changing $u_i$ will not change the conditional distribution of $\sigma_v$, which means $R_{vu_i} = 0$. Now consider $R_{u_iv}$. Notice that $v$ is the only neighbor of $u_i$, hence it is the only one that matters in the conditional distribution. If $\sigma_v = 1$, then the conditional distribution of $\sigma_{u_i}$ is uniform on $\{2,\ldots,q+1\}$. 
If $\sigma_v > 1$, then $\sigma_u =1$ necessarily. Thus, $R_{u_iv} = 1 - 1/q$. We conclude that this distribution satisfies Dobrushin's condition with $\alpha = 1 - 1/q$, regardless of the value of $\beta$. 
\end{proof}

Our next goal is to investigate whether the Dobrushin-Shlosman's condition is sufficient for learning.
\begin{theorem}
[Dobrushin-Shloshman does not suffice]
\label{thm:ds insufficient}
For any $n \in \nats$,
there exist graphs $(G = G_n, H = H_q)$ and parameters $\beta_1 < \beta_2$ where $\beta_2 - \beta_1$ is a sufficiently small constant 
such that
the H-coloring model
$\Pr_{G, H, \beta}$
satisfies Dobrushin-Shlosman's condition
for any $\beta \in (\beta_1,\beta_2)$
but for any $v \in V(G_n)$ and any $\beta$, it holds that
$ \Var_{\sigma \sim {G,H,\beta}}(c_1(\sigma) ) = 0$. 
Consequently, $\dkl(P_{\beta_1} \parallel P_{\beta_2}) = 0$.
\end{theorem}
\begin{proof}
Let $G_n$ be a cycle with $2n$ nodes $u_1,\ldots,u_{2n}$. 
Suppose $q = \{1,2,3\}$ and $H$ has edge set $E = \{(1,2), (1,3)\}$. Let $A = \{u_1,u_3,\ldots\}, B = \{u_2,u_4,\ldots\}$ be the even and odd vertices. 
Then, it is clear that the valid $H$-colorings of $G_n$ will have $\sigma_A = 1$ and $\sigma_B \in \{2,3\}$ or $\sigma_B = 1$ and $\sigma_A \in \{2,3\}$. 
Suppose the Hamiltonian is of the form $\beta c_1(\sigma)$. Then, all valid $H$-colorings have exactly $n$ vertices of color $1$. Thus $\Var[c_1(\sigma)] = 0$ for all $\beta$, which implies $\dkl(P_{\beta_1} \parallel P_{\beta_2}) = 0$ by Lemma~\ref{lemma:exp family kl}.
Let us calculate the Dobrushin matrix. The only non-zero entries have the form $R_{u_iu_{i+1}}$. Notice that the value of $u_{i-1}$ completely determines the conditional distribution of $\sigma_{u_i}$ regardless of the value of $\sigma_{u_{i+1}}$. Thus, $R_{u_iu_{i+1}} = 0$, which means that both Dobrushin and Dobrushin-Shlossman are satisfied in this example.  
\end{proof}

Finally we will investigate the Rainbow condition proposed by \cite{bhattacharya2021parameter}.
\begin{definition}
[Average-degree \& Rainbow Condition]
\label{cond:1}
The sequence of graphs $(G_n)_{n \in \nats}$
has uniformly bounded average degree. Moreover, for any color $r\in [q-1]$ of $H$, it holds that
\begin{equation}\label{eq:rainbow}
\lim_{\eps \to 0} \liminf_{n \to \infty}
\Pr_{\sigma \sim G_n, H, \beta}[ u_r^n(\sigma) > \eps n] = 1\,,
\end{equation}
where $u_r^n(\sigma)$ is the number of vertices  in $G_n$
whose neighbors are assigned colors in $\sigma$ that lie in the $H$-neighborhood of $r$ (i.e., a vertex $v$ contributes to $u_r^n(\sigma)$ if it can be recolored to $r$ while keeping the colors of other vertices the same).
\end{definition}
This condition was investigated by \cite{bhattacharya2021parameter}. We will show that this condition is not sufficient, already in the case of proper $q$-colorings on maximum degree $d$ graphs (note that the Rainbow Condition is easily seen to be satisfied when $q \geq d+1$).
\begin{theorem}
[Regime $q \leq d+1$]
\label{t:q_coloring_2}
Fix $q \geq 2$ and let $\pmb{\beta^*} = (\beta_1^*,\ldots,\beta_{q}^*) \in \R^{q}$ with $\beta_q^*=0$.

There exists a sequence of graphs $\{G_n\}_{n \in \mathbb{N}}$ with maximum degree $d$ and number of colors $q = d+1$, such that for any $\pmb{\beta_1}, \pmb{\beta_2} \in \R^{q}$, $\dkl \left(
    \Pr_{G_n, \pmb{\beta_1}} \parallel \Pr_{G_n, \pmb{\beta_2}}\right)=0$.
\end{theorem}
\begin{proof}
For a vector $\pmb \beta \in \mathbb{R}^{q}$ (with $\beta_q = 0$), define the free-energy function
$$
F(\pmb \beta) = \ln \bigg(\sum_{\sigma \in \Omega_{G_n}} e^{\sum_{s = 1}^{q-1} \beta_s c_s(\sigma)} \bigg).
$$
We also use the notation $P_{\pmb \beta}$ for the distribution with parameter $\pmb \beta$ and $\sigma \sim \pmb \beta$ when $\sigma$ is sampled from that distribution. 
We can use \Cref{lemma:exp family kl} which relates the KL divergence between two hard constraint models with different parameters. We get that, for 
$\pmb \beta, \pmb\gamma \in \reals^{q}$, it holds that 
\begin{equation}\label{eq:v5g566h6g}
\dkl(P_{\pmb \beta} \parallel P_{\pmb \gamma}) = \frac{1}{2} \Var_{\sigma\sim \pmb\beta + \xi( \pmb\gamma - \pmb\beta)} \lp[\sum_{s=1}^{q-1} (\gamma_s - \beta_s) c_s(\sigma)\rp]\,,
\end{equation}
for some $\xi \in (0,1)$. We next define such a family of graphs where this variance is $0$.

Assume $n$ is a multiple of $q$, and let $G_n$ be a graph consisting of $n/q$ cliques of size $q$.  Clearly, for each coloring $\sigma$ of $G_n$ and for any $r \in [q]$ we have $c_r(\sigma) = n/q$. Hence, the random variable 
$\sum_{s=1}^{q-1} (\gamma_s - \beta_s)c_s(\sigma)
$ appearing in \eqref{eq:v5g566h6g} is a constant, regardless of the values of the parameters $\pmb \beta$ of the model.
This gives that $\dkl(P_{\pmb \beta} \parallel P_{\pmb \gamma}) = 0$ for any $\pmb \beta, \pmb\gamma$ and, hence, estimation is impossible for this family of graphs.
\end{proof}

\section{Parameter Estimation for the $H$-Coloring Model}
\label{appendix:h-coloring}
In this section, we prove \Cref{theorem:positive result h-colorings} and then Corollaries~\ref{t:q_coloring} and ~\ref{cor:permissive}.
\subsection{Algorithm for the $H$-Coloring Model Under 2-Rainbow Condition}
\label{sec:algo-h-coloring}
Let $G$ be a graph with max degree $d$ and vertex set $[n]$. Let $H$ be a graph with vertex set $[q]$ and let $\pmb{\beta}^\star = (\beta_1^\star,\ldots,\beta_{q}^\star) \in \R^{q}$ be a vector of parameters with $\beta_{q}^\star=0$ (cf. \Cref{fn:betaqzero}). Our goal is to learn $(\beta_1^*,\hdots,\beta_{q-1}^*)$ from a single sample $\sigma\sim \Pr_{G,H,\bm \beta^*}$. We assume that $H$ satisfies the 2-Rainbow condition on max degree $d$ graphs (though this will not  be important until the proof of \Cref{theorem:positive result h-colorings}.)

To do this, we utilize the Maximum Pseudo-Likelihood Estimator (MPL) (cf. \Cref{def:mpl}), which was formulated in the $H$-coloring setting by \cite{bhattacharya2021parameter}. Namely, they showed that the first derivative of the negative pseudo-likelihood function $\phi(\pmb \beta;\sigma)$ is  given by $\phi_1(\pmb{\beta};\sigma)=(\phi_1^{(r)}(\pmb{\beta};\sigma))_{r\in [q]}$, where for a color $r\in [q]$ 
\begin{align*}
\phi_1^{(r)}(\pmb{\beta};\sigma) =  c_r(\sigma) - \sum_{u=1}^n \frac{e^{\beta_r}Q(v,r)}{\sum_{s=1}^{q} e^{\beta_s}Q(v,s)},
\end{align*}
where, recall from \Cref{sec:4g4gg4}, for a vertex $v \in [n]$,  $s \in [q]$ and $\sigma\in \Omega_G$, we set $Q_\sigma(v,s)=1$ if assigning the color $s$ to $v$ while keeping the colors of other nodes according to $\sigma$ results in a proper coloring, and otherwise we set $Q_\sigma(u,s) = 0$. The following lemma  was proved in \cite{bhattacharya2021parameter} and upper-bounds the first derivative when evaluated at the parameter $\pmb{\beta^\star}$.
\begin{lemma}[{\cite[Lemma 5.1]{bhattacharya2021parameter}}]
\label{lemma:proper coloring first derivative}
Let $d\geq 2$ be an integer, $B>0$ be a real and $H$ be a $q$-vertex graph. There exists $L>0$ such that  for any $n$-vertex graph $G$ with max degree $d$ and  $\Omega^H_G\neq \emptyset$, for $\pmb\beta=(\beta_1,\hdots,\beta) \in [-B,B]^q$ with $\beta_q=0$ it holds that $\E\big[\|\phi_1(\beta; \sigma)\|^2_2\big]\leq L n$ where the expectation is over $\sigma\sim \Pr_{G,H,\pmb\beta}$. 
\end{lemma}
We next proceed to formulate the second derivative of $\phi$. Recall from \Cref{sec:4g4gg4} that $O_\sigma(v) = (\theta_\sigma(u,1),\ldots,\theta_\sigma(u,q-1))^\top$ denotes the $(q-1)$-dimensional vector with $
\theta_\sigma(u,r) = \frac{\emm^{\beta_r}Q_\sigma(v,r)}{\sum_{s=1}^q\emm^{\beta_s}Q_\sigma(v,s)}$. Then, it holds that (see \cite{bhattacharya2021parameter} for details)
\begin{equation*}\tag{\ref{eq:hcol_second_der}}
\phi_2(\pmb{\beta};\sigma) = \sum_{v=1}^n \lp(\diag(O_\sigma(v)) - O_\sigma(v)O_\sigma(v)^\top \rp).
\end{equation*}
We next give the proof of the inequality in \Cref{l:eigenvalue} which we restate here for convenience.
\direction*
\begin{proof}
By the Cauchy-Schwarz inequality, we have
\begin{align*}
\bigg(\sum_{s=1}^{q-1} \theta_\sigma(v,s)x_s\bigg)^2 &\leq \bigg(\sum_{s=1}^{q-1} \theta_\sigma(v,s)x_s^2\bigg)\bigg(\sum_{s=1}^{q-1}  \theta_\sigma(v,s)\bigg)= \bigg(\sum_{s=1}^{q-1} \theta_\sigma(v,s)x_s^2\bigg)\bigg(1 - \frac{\emm^{\beta_q}}{\sum_{s=1}^{q} \emm^{\beta_s}Q_\sigma(v,s)}\bigg).
\end{align*}
\noindent Thus, using the fact that $\beta_q = 0$,
\begin{align*}
\mathbf{x}^\top \lp(\diag(O_\sigma(v)) - O_\sigma(v)O_\sigma(v)^\top \rp) \mathbf{x} &= \sum_{s=1}^{q-1} \theta_\sigma(v,s)x_s^2 - \bigg(\sum_{s=1}^{q-1} \theta_\sigma(v,s)x_s\bigg)^2\geq \frac{\min_{s \in [q]} \emm^{\beta_s}}{\sum_s \emm^{\beta_s}} \bigg(\sum_{s=1}^{q-1} \theta_\sigma(v,s)x_s^2\bigg).
\end{align*}
Now, we have
\[
\sum_{s=1}^{q-1} \theta_\sigma(v,s)x_s^2 \geq \sum_{r \in \mathcal{I}} \theta_\sigma(v,r)x_r^2 \geq \min_{s\in [q-1]}\theta_\sigma(v,s) \sum_{r \in \mathcal{I}(\sigma)} x_r^2,
\]
where the last inequality follows since we know that $\theta_\sigma(v,r) \neq 0$ for all $r \in \mathcal{I}(\sigma)$. After observing that $\min_{s\in [q-1]}\theta_\sigma(v,s) \geq \frac{\min_{s \in [q]} \emm^{\beta_s}}{\sum_s \emm^{\beta_s}Q_\sigma(v,s)} \geq 
\frac{\min_{s \in [q]} \emm^{\beta_s}}{\sum_s \emm^{\beta_s}}
$, we therefore obtain the inequality in the lemma.
\end{proof}
Hence, each term in the second derivative sum preserves some directions of a vector $\mathbf{x}$. 
We now use this to prove that $\phi$ is strongly convex. Recall that 
$u_{rq}(\sigma) = \sum_{i=1}^n Q_{\sigma
}(i,r)Q_{\sigma}(i,q)$.
\begin{lemma}\label{l:strong_conv}
Let $\sigma \in \Omega_{G}^H$ and suppose there exists $\delta > 0$ such that  $u_{r q}(\sigma) \geq \delta n$ for all $r \in [q-1]$. 
Then, $\phi_2(\pmb \beta;\sigma) \succcurlyeq \delta n \Big(\frac{\min_{s \in [q]} \emm^{\beta_s}}{\sum_{s=1}^q \emm^{\beta_s}}\Big)^2\mathbf{I}_{q-1}.
$
\end{lemma}
\begin{proof}
Let $\mathbf{x} \in \R^{q-1}$ be an arbitrary vector.
We have that
\begin{align*}
\mathbf{x}^\top \phi_2(\pmb\beta;\sigma)\mathbf{x} &= 
\sum_{v=1}^n \mathbf{x}^\top \lp(\diag(O_\sigma(u)) - O_\sigma(u)O_\sigma(u)^\top \rp)\mathbf{x}\,.
\end{align*}
Define $\mathcal{I}_v(\sigma)=  \{r \in [q-1] : Q_\sigma(v,r) = 1\}$,
and 
$U_q = \{ v \in [n]: Q_\sigma(v,q) = 1\}$.
Then, since every matrix in the sum is positive definite, by applying \Cref{l:eigenvalue}, we have  that 
\begin{align*}
\mathbf{x}^\top \phi_2(\pmb\beta;\sigma)\mathbf{x} \geq \sum_{v \in U_q} \mathbf{x}^\top \lp(\diag(O_\sigma(v)) - O_\sigma(v)O_\sigma(v)^\top \rp)\mathbf{x}\geq  
\bigg(\frac{\min_{s \in [q]} \emm^{\beta_s}}{\sum_{s=1}^q \emm^{\beta_s}}\bigg)^2
\sum_{v \in U_q} \sum_{r \in \mathcal{I}_v} x_r^2.
\end{align*}
Now, by definition $u_{rq}(\sigma)=\lp|\{v\in V : v \in U_q, \mathcal{I}_v(\sigma) \ni r\}\rp|$. Thus, 
\begin{align*}
\mathbf{x}^\top \phi_2(\pmb\beta;\sigma)\mathbf{x}\geq  \lp(\frac{\min_{s \in [q]} \emm^{\beta_s}}{\sum_{s=1}^q \emm^{\beta_s}}\rp)^2\sum_{r=1}^{q-1} u_{rq}(\sigma)x_r^2
\geq \delta n \lp(\frac{\min_{s \in [q]} \emm^{\beta_s}}{\sum_{s=1}^q \emm^{\beta_s}}\rp)^2 \|\mathbf{x}\|_2^2,
\end{align*}
finishing the proof.
\end{proof}

We next give the proof of \Cref{theorem:positive result h-colorings} which we restate here for convenience. Recall that the 2-Rainbow Condition asserts that there is a constant $\delta>0$ such that 
\begin{equation*}\tag{\ref{eq:strong-rainbow}}
\Pr_{G, H, \pmb \beta^*}
\big[u_{rq}(\sigma) \geq \delta n, ~\forall r \in [q-1]\big] = 1-o(1).
\end{equation*}

\hcolpositive*
\begin{proof}
For $t \in [0,1]$, let $\pmb{\beta}(t) = t\hat{\pmb{\beta}}_\sigma + (1-t)\pmb{\beta^\star}$ and $h(t) = (\pmb{\hat{\beta}}_\sigma - \pmb{\beta^\star})\phi_1(\pmb{\beta}(t);\sigma)$. Since $\hat{\pmb{\beta}}_\sigma$ maximizes $\phi$, we have that $\phi_1(\hat{\pmb{\beta}};\sigma)=0$ and hence $h(1)=0$.

For any $\pmb \beta\in [-2B,2B]^q$, we have the crude bound  $\lp(\frac{\min_{s \in [q]} \emm^{\beta_s}}{\sum_{s=1}^q \emm^{\beta_s}}\rp)^2 \geq C=1/(q^2\emm^{8B})$. Define the set of colorings 
\[
\mathcal{A} = 
\left\{\sigma \in \Omega_{G}^H : \|\phi_1(\pmb \beta;\sigma)\|_2 \leq \sqrt{2Ln/\epsilon} ~~\text{ and } ~ \min_{\pmb \beta\in [-2B,2B]^q}\lambda_{\min}(\phi_2(\pmb\beta;\sigma) \geq C\delta n \right\},
\]
and note that, from Lemmas~\ref{lemma:proper coloring first derivative} and~\ref{l:strong_conv} (combined with a union bound), it holds that $\Pr_{G,H,\pmb \beta^*}[A]\geq 1-\epsilon/2-o(1)\geq 1-\epsilon$ for all sufficiently large $n$.

Consider $\sigma\in\mathcal{A}$. We have from Cauchy-Schwarz inequality that
\begin{align*}
\big\|\hat{\pmb{\beta}}_\sigma &- \pmb{\beta^\star}\big\|_2 \big\|\phi_1(\pmb{\beta^\star};\sigma)\big\|_2\\
&\geq \big|(\pmb{\hat{\beta}}_\sigma - \pmb{\beta^\star})L(\pmb{\beta^\star})\big| = |h_n(1) - h_n(0)|= \lp|\int_0^1 h_n'(t) \mathrm{d}t \rp| = 
\lp|\int_0^1 (\hat{\pmb{\beta}}_\sigma - \pmb{\beta^\star})^\top \phi_2(\pmb{\beta}(t);\sigma)(\hat{\pmb{\beta}}_\sigma - \pmb{\beta^\star}) \mathrm{d}t \rp|
\end{align*}
Note that $\phi_2(\pmb\beta;\sigma)\succcurlyeq 0$ for every $\beta$ (cf. \eqref{eq:hcol_second_der}), so with $\ell = \|\hat{\pmb{\beta}}_\sigma - \pmb{\beta^\star}\|_2$   we obtain
\begin{equation}\label{eq:g6y6yhyht3er}\big\|\hat{\pmb{\beta}}_\sigma -  \pmb{\beta^\star}\big\|_2 \big\|\phi_1(\pmb{\beta^\star};\sigma)\big\|_2\
\geq 
\bigg|\int_0^{\min\{1,B/\ell\}} (\hat{\pmb{\beta}}_\sigma - \pmb{\beta^\star})^\top \phi_2(\pmb{\beta}(t);\sigma)(\hat{\pmb{\beta}}_\sigma - \pmb{\beta^\star}) \mathrm{d}t \bigg|.
\end{equation}
For $0\leq t\leq B/\ell$, we have that $\pmb\beta_\sigma(t)\in [-2B,2B]^{q}$. Hence, since $\sigma \in \calA$, \eqref{eq:g6y6yhyht3er} gives that $
\sqrt{2Ln/\eps} \geq C\delta n \min\{\ell, B\}$. Therefore, for all sufficiently large $n$, we obtain that  $ \|\hat{\pmb{\beta}}_\sigma - \pmb{\beta^\star}\|_2=\ell\leq M/\sqrt{n}$, where $M = \sqrt{2L/C^2\delta^2\eps}$, finishing the proof.
\end{proof}

\subsection{Application of \Cref{theorem:positive result h-colorings} to proper $q$-colorings}
\label{sec:app-1}
Let $\pmb{\beta^\star}\in \mathbb{R}^q$ and $G$ be a graph of max degree $d$ with vertex set $[n]$. In order to prove the positive part of \Cref{t:q_coloring}, we need to establish the 2-Rainbow Condition (cf. \eqref{eq:strong-rainbow}).  For a coloring $\sigma \sim \Pr_{G, \pmb{\beta^\star}}$ and $i \in [n], r\in [q]$, define the random variable $e^{rq}_i(\sigma)=Q(i,q)Q(i,r)$, in other words, $e^{rq}_i(\sigma) = 1$ if assigning $\sigma_i = r$ or $\sigma_i = q$ while keeping the other vertices colored according to $\sigma$ results in a proper coloring. By definition of $u_{rq}(\sigma)$, we have $u_{rq}(\sigma) = \sum_{i=1}^n e_i^{rq}(\sigma)$.

Let $R\subseteq V$ be a maximal subset of nodes that forms a 2-hop independent set of $G$, i.e., distinct $v,w \in R$ have disjoint neighborhoods. We can pick $R$ greedily; using that $G$ has maximum degree $d$, we have that $|R|\geq n/(d+1)^2$.  
\couplingcolorings*
\begin{proof}
Let $\Omega_x$ be the subset of proper $q$-colorings that are consistent with $x=(x_1,\hdots, x_{k-1})$. Let $\Omega_x^+$ be the subset of $\Omega_x$ containing all $\sigma$ with $e^{rq}_{v_k}(\sigma) = 1$ and similarly, $\Omega_x^-$ be the subset  containing all $\sigma$ with $e^{rq}_{v_k}(\sigma) = 0$.

We will construct a mapping $g:\Omega_x^- \mapsto \Omega_x^+$ as follows. Let $\sigma \in \Omega_x^-$ and let $\mathcal{N}_{v_k}$ be the set of neighbors of $v_k$ in $G$. Denote by $B$ the neighbors $j \in \mathcal{N}_{v_k}$ such that $\sigma_j = r$ or $\sigma_j = q$. Since $e^{rq}_{v_k}(\sigma) = 0$, $B$ is nonempty. We will produce a proper $q$-coloring $\sigma'$ by changing the colors of the nodes in $B$. To do this, first erase the colors of the nodes in $B$, while keeping the rest of the colors as in $\sigma$.  We will iteratively assign colors to the nodes in $B$, each time ensuring that the color chosen does not appear in any neighbors of that node that have not been erased. Pick a node $j \in B$. Let $m = |\mathcal{N}_{j}\cap B|$ and $m' =  |\mathcal{N}_{j}\cap (B\cup\{v_k\})^c|$. Since $v_k$ is a neighbor of $j$, we have that $m+m'+1 \leq d $, which means that $m' \leq d-1-m \leq q-3-m$ . Thus, we have at least $m+3\geq 3$ available colors for $j$ that do not conflict with any colors assigned to vertices in $\mathcal{N}_{j}\cap (B\cup\{v_k\})^c$. Let $s$ be one of these colors that is not equal to $r$ or $q$. We set $\sigma'(j) = s$ and remove $j$ from $B$. This process can be carried out for all nodes in $B$, each time coloring a node with a color different than $q$ and $r$, while maintaining the property that the coloring produced so far is proper. Thus, in the end, we set $\sigma'(v_k) = q$ and $g(\sigma) = \sigma'$, and by definition no neighbor of $v_k$ is colored with $r$ or $q$ in $\sigma'$. Furthermore, we only changed the colors of neighbors of $v_k$. Since $v_1,\ldots,v_{k-1}$ do not share any neighbors with $v_k$ (by definition of the set $R$), $\sigma'$ is still consistent with the assignment $e_{v_1}^{rq}(\sigma) = x_1,\ldots,e_{v_{k-1}}^{rq}(\sigma) = x_{k-1}$. In other words, we just showed that $g(\sigma) \in \Omega_x^+$. 

Since $g$ only changes the colors of $v_k$ and its neighbors, we have that for any $\sigma' \in \Omega_x^+$, 
the subset of $\Omega_x^-$ that is mapped to $\sigma'$ under $g$ has size
$|g^{-1}(\sigma')| \leq q^{d+1}$. Furthermore, if $\sigma' = g(\sigma)$, then
\[
\exp\bigg(\sum_{s=1}^q
\beta_s^\star
c_s(\sigma)\bigg) \leq \exp(\|\pmb{\beta^\star}\|_1 d) \exp\bigg(\sum_{s=1}^q 
\beta_s^\star
c_s(\sigma')\bigg).
\]
These observations yield that
\begin{align*}
     \Pr_{G_n, \pmb{\beta^\star}}[e_{v_k} = 1|e_{v_1}^{rq}(\sigma) = x_1,\ldots,e_{v_{k-1}}^{rq}(\sigma) = x_{k-1} ] 
     &= 
     \frac{\sum_{\sigma \in \Omega_x^+} \exp\lp(\sum_{s=1}^q 
     \beta_s^\star
     c_s(\sigma)\rp) }{\sum_{\sigma \in \Omega_x^+} \exp\lp(\sum_{s=1}^q
     \beta_s^\star c_s(\sigma)\rp)+\sum_{\sigma \in \Omega_x^-} \exp\lp(\sum_{s=1}^q
     \beta_s^\star
     c_s(\sigma)\rp)}
     \,,
\end{align*}
and the right-hand side is at least
\[
\frac{\sum_{\sigma \in \Omega_x^+} \exp\lp(\sum_{s=1}^q 
\beta_s^\star
c_s(\sigma)\rp) }{\sum_{\sigma \in \Omega_x^+} \exp\lp(\sum_{s=1}^q
\beta_s^\star
c_s(\sigma)\rp)+q^d e^{\|\pmb{\beta^\star}\|_1 d}\sum_{\sigma \in \Omega_x^+} \exp\lp(\sum_{s=1}^q 
\beta_s^\star
c_s(\sigma)\rp)}
 \geq \frac{1}{1+q^{d+1}e^{\|\pmb{\beta}^\star\|_1 d}}.\qedhere
\]
\end{proof}

We are now ready to complete the proof of \Cref{t:q_coloring}, which we restate here. 
\pestimatecol*
\begin{proof}
 The impossibility part for $q\leq d+1$ has already been shown in~\Cref{t:q_coloring_2} in \Cref{sec:insufficient}. So we focus on the statement for the existence of the estimator when $q\geq d+2$.
 
From \Cref{theorem:positive result h-colorings}, it suffices to show that in the regime $q \geq d+2$, the 2-Rainbow condition (cf. \eqref{eq:strong-rainbow}) is satisfied on graphs $G$ of max degree $d$, i.e., we want to show that there exists $\delta>0$ such that for all $r \in [q-1]$, it holds that
$ \Pr_{G,\pmb{\beta}^\star}[u_{rq}(\sigma) \geq  \delta n] = 1-o(1)$, where 
$u_{rq}(\sigma)$ is the number of vertices that can be colored by either $r$ or $q$ while still being a proper $q$-coloring.

Let $R = \{v_1,....,v_{t}\}$ be a maximal 2-hop independent set  (i.e., the vertices in $R$ have mutually disjoint neighborhoods), so that $t\geq n/(d+1)^2$. Then, for distinct colors $r$ and $q$, we have the lower bound $u_{rq}(\sigma)\geq \sum^{t}_{k=1}e^{rq}_{v_k}(\sigma)$. By the ``conditional independence'' property of \Cref{lemma:coupling-lower bound}, we have by an application of \Cref{lem:dominance} that the sum stochastically dominates the sum $\sum^t_{i=1} \tau_i$ where $\{\tau_i\}$ is a family of 0-1 random variables with $\Pr[\tau_i=1]=p$, where $p=\frac{1}{1 + q^{d+1} \emm^{\|\pmb \beta^\star\|_1 d}}$. By standard Chernoff bounds, for $\delta=p/(4d)^2$ it follows that $u_{rq}(\sigma)\geq \delta n$ with probability $1-\emm^{-\Omega(n)}$ over the choice of $\sigma$. This completes the proof.
\end{proof}

\subsection{Application of \Cref{theorem:positive result h-colorings} to Permissive Models}
\label{sec:app-2}
In this section, we prove \Cref{cor:permissive}.
\corpermissive*
\begin{proof}
As in the case of proper $q$-colorings, the result will follow by applying \Cref{theorem:positive result h-colorings}. We only need to establish the 2-Rainbow Condition, i.e., that for any color $r$ it holds that $u_{rq}(\sigma)=\Omega(n)$ with probability $1-o(1)$ over the choice of $\sigma$. The proof of this is very similar to that of \Cref{t:q_coloring}, once we prove a suitable analogue of the ``conditional independence'' property of \Cref{lemma:coupling-lower bound}. 

Let $G=(V,E)$. Suppose that the unconstrained color of $H$ is $j$. For a vertex $v$ of $G$, the key is to consider the indicator $e_v(\sigma)$ of the event that all neighbors of $v$ get the unconstrained color $j$. Then, since color $j$ is unconstrained, for any color $r\in [q]$, we have the lower bound $u_{rq}(\sigma)\geq \sum_{v\in V}e_v(\sigma)$. It remains therefore to lower bound the latter sum.

Analogously now to proper colorings, consider a maximal set $R=\{v_1,\hdots,v_t\}$ of vertices with mutually disjoint neighborhoods; we have that $t\geq n/(d+1)^2$. Then, for $k=1,\hdots, t$, it holds that
\begin{equation}\label{eq:4ttg4tgwq}\Pr_{G,H,\bm \beta^*}[e_{v_k}(\sigma)=1\mid e_{v_1},\hdots, e_{v_{k-1}}(\sigma)]\geq p^d\,,~ \mbox{ where ~} p=\frac{e^{\beta_j^\star}}{\sum_{s=1}^q e^{\beta_s^\star}}.
\end{equation}
To prove \eqref{eq:4ttg4tgwq}, we will show that for any vertex $w$, conditioned on any coloring of its neighborhood, the probability that $w$ gets the unconstrained color $j$ is at least $p$. Assuming this for the moment, the probability that all neighbours of $v_k$ get $j$  is at least $p^d$. Since the conditioning on $e_{v_1},\hdots, e_{v_{k-1}}$ is completely determined by the coloring of the neighborhoods of $v_1,\hdots,v_{k-1}$ and $v_k$ does not have a neighbor in any of these (by the choice of $R$), the bound on the conditional probability in \eqref{eq:4ttg4tgwq} follows. Using \eqref{eq:4ttg4tgwq}, we can now invoke~\Cref{lem:dominance} analogously to the proper $q$-colorings case to conclude that $u_{rq}(\sigma)\geq \sum^{t}_{i=1}e_{v_i}(\sigma)\geq \delta n$ with probability $1-\emm^{-\Omega(n)}$ over the choice of $\sigma$, where $\delta=p^d/(8d)^2$. 

To finish the proof of the 2-Rainbow Condition, it remains to show the intermediate fact that we used in the proof of \eqref{eq:4ttg4tgwq}, i.e.,  for a vertex $w$ with neighbors $w_1,\hdots, w_d$, it holds that
\begin{equation}\label{eq:keyHcol}
\Pr_{G,H,\bm \beta^*}[\sigma_w=j\mid \sigma_{w_1}=x_1,\hdots, \sigma_{w_d}=x_d]\geq p
\end{equation}
for any choice of colors $x_1,\hdots,x_d$ (which are consistent with some $H$-coloring of $G$). To see \eqref{eq:keyHcol}, let $x=(x_1,\hdots,x_d)$ and denote by $\Omega_{x}$ the set of all valid $H$-colorings $\sigma$ such that $\sigma_{w_1}=x_1,\hdots, \sigma_{w_d}=x_d$.  For $s \in [q]$, denote further by $\Omega^s_{x}$ those $H$-colorings in $\Omega_{x}$ where $\sigma_w = s$, so that $\{\Omega^s_{x}\}_{s=1}^q$ is a partition of $\Omega_{x}$. It follows that
\begin{equation}\label{eq:4f4fttgrtef}\Pr_{G,H,\bm \beta^*}[\sigma_w=j\mid \sigma_{w_1}=x_1,\hdots, \sigma_{w_d}=x_d]=\frac{\sum_{\sigma \in \Omega^j_{x}} \exp\lp(\sum_{i=1}^{q-1} \beta_i^\star c_i(\sigma)\rp)}{\sum_{s=1}^q\sum_{\sigma \in \Omega^s_{x}} \exp\lp(\sum_{i=1}^{q-1} \beta_i^\star c_i(\sigma)\rp)}.
\end{equation}
Now, for any $s \neq j$, define the mapping $g_s:\Omega^s_{x} \mapsto \Omega^j_{x} $ by setting for $v\in V$
$$
g_s(\sigma)_v = \left\{
\begin{array}{ll}
      j &\text{ , if $v  =w$},\\
      \sigma_v &\text{ , otherwise.}
\end{array} 
\right. 
$$
Since the color $j$ is unconstrained, $g(\sigma)$ is a valid $H$-coloring and $g_s(\sigma) \in \Omega^j_{x}$; also $g_s$ is  injective. Now, suppose that $g_s(\sigma) = \sigma'$,  we have
\[
\exp\bigg(\sum_{i=1}^{q-1} \beta_i^\star c_i(\sigma)\bigg) = \exp(\beta^\star_s - \beta^\star_j)
\exp\bigg(\sum_{i=1}^{q-1} \beta_i^\star c_i(\sigma')\bigg),
\]
since $\sigma$ and $\sigma'$ only differ in the color of node $k$. Using the injectivity of $g_s$, it thererefore follows that
$$
\sum_{\sigma \in \Omega^s_{x} } \exp\bigg(\sum_{i=1}^{q-1} \beta_i^\star c_i(\sigma)\bigg) \leq 
\exp(\beta^\star_s - \beta^\star_r)
\sum_{\sigma \in \Omega^j_{x }} \exp\bigg(\sum_{i=1}^{q-1} \beta_i^\star c_i(\sigma)\bigg).
$$
Using this to upper-bound the sums in the denominator of \eqref{eq:4f4fttgrtef} for all $s \neq j$, we obtain \eqref{eq:keyHcol} after observing that $p=\frac{1}{1 +\sum_{s\neq j}\exp(\beta^\star_s - \beta^\star_j)}$. This finishes the proof of \Cref{cor:permissive}. 
\end{proof}

\end{document}